\documentclass[10pt,journal,cspaper,compsoc]{IEEEtran}
\usepackage[utf8]{inputenc} % allow utf-8 input
\usepackage[T1]{fontenc}    % use 8-bit T1 fonts
\usepackage{url}            % simple URL typesetting
\usepackage{booktabs}       % professional-quality tables
\usepackage{amsfonts}       % blackboard math symbols
\usepackage{nicefrac}       % compact symbols for 1/2, etc.
\usepackage{microtype}      % microtypography
\usepackage{amssymb}
\usepackage{nomencl}
\makenomenclature

\usepackage{amsmath}
\usepackage{amsthm,subfigure}
\usepackage{graphicx}
\usepackage{algpseudocode}
\usepackage[linesnumbered,ruled]{algorithm2e}
\usepackage{xfrac}
\usepackage{xcolor}
\usepackage{wrapfig}
\usepackage{paralist}
\usepackage{mathtools}
\usepackage{pgfplots}
\usepackage{comment}
\usepackage{tikz}
\usepackage{hyperref}
\usepackage[customcolors,shade]{hf-tikz}
\newcommand{\subparagraph}{}
\usepackage{empheq,multicol,titlesec}
\usepackage[inline]{enumitem}

\usepackage[most]{tcolorbox}
\newtcbox{\mymath}[1][]{%
    nobeforeafter, math upper, tcbox raise base,
    enhanced, colframe=blue!30!black,
    colback=gray!30, boxrule=1pt,
    #1}

\usepackage{multirow}
\usepackage{colortbl, array}
\usepackage{hhline}

\newcommand{\highest}[1]{\textcolor{black}{\mathbf{#1}}}

\definecolor{rulecolor}{RGB}{0,71,171}
\definecolor{tableheadcolor}{RGB}{204,229,255}

\newtheorem{theorem}{Theorem}

\newtheorem{definition}{Definition}
\newtheorem{proposition}{Proposition}

\newtheorem*{theorem*}{Theorem}
\newtheorem*{definition*}{Definition}

\makeatletter
\newsavebox{\mybox}\newsavebox{\mysim}
\newcommand{\distras}[1]{%
  \savebox{\mybox}{\hbox{\kern3pt$\scriptstyle#1$\kern3pt}}%
  \savebox{\mysim}{\hbox{$\sim$}}%
  \mathbin{\overset{#1}{\kern\z@\resizebox{\wd\mybox}{\ht\mysim}{$\sim$}}}%
}

\definecolor{blue}{rgb}{0., 0., 0.}
\ifCLASSOPTIONcompsoc
  \usepackage[nocompress]{cite}
\else
  \usepackage{cite}
\fi
\ifCLASSINFOpdf
\else
\fi
\begin{document}

\title{VolterraNet: A higher order convolutional network with group equivariance for homogeneous manifolds}

\author{Monami Banerjee$^\dagger$, Rudrasis~Chakraborty$^\dagger$, Jose Bouza,  and~Baba~C.~Vemuri,~\IEEEmembership{Fellow,~IEEE}
\IEEEcompsocitemizethanks{ \IEEEcompsocthanksitem M. Banerjee is with FaceBook, CA,
R. Chakraborty is with University of California, Berkeley, USA. $^\dagger$ denotes equal contributions. \IEEEcompsocthanksitem J. Bouza and B. C. Vemuri are with University of Florida, Gainesville, FL, USA.\protect\\
E-mail: \{monamie.b, rudrasischa, josejbouza\}@gmail.com, vemuri@ufl.edu }% <-this % stops an unwanted
}



\markboth{ACCEPTED IN IEEE TRANSACTIONS ON PATTERN ANALYSIS AND MACHINE INTELLIGENCE,~VOL.~XXX, NO.~XXX, MONTH~YEAR}%
{Chakraborty \MakeLowercase{\textit{et al.}}: VolterraNet -- A Higher Order CNN for Riemannian Homogeneous Manifolds .}



\IEEEtitleabstractindextext{%

\begin{abstract}
Convolutional neural networks have been highly successful in image-based learning tasks due to their translation equivariance property.  Recent work has generalized the traditional convolutional layer of a convolutional neural network to non-Euclidean spaces and shown group equivariance of the generalized convolution operation. 
In this paper, we present a novel higher order Volterra convolutional neural network (VolterraNet) for data defined as samples of functions on Riemannian homogeneous spaces.
Analagous to the result for traditional convolutions, we prove that the Volterra functional convolutions are equivariant to the action of the isometry group admitted by the
Riemannian homogeneous spaces, and under some restrictions, any non-linear equivariant function can be expressed as our homogeneous space Volterra convolution, generalizing the non-linear shift equivariant characterization of Volterra expansions in Euclidean space.  We also prove that second order functional convolution operations can be represented as cascaded convolutions which leads to an efficient implementation. Beyond this, we also propose a dilated VolterraNet model. These advances lead to large parameter reductions relative to baseline non-Euclidean CNNs.

To demonstrate the efficacy of the VolterraNet performance, we present several real data experiments involving classification tasks on spherical-MNIST, atomic energy, Shrec17 data sets, and group testing on diffusion MRI data. Performance comparisons to the state-of-the-art are also presented. 

\end{abstract}

}

\maketitle

\IEEEdisplaynontitleabstractindextext

% \fi
\IEEEpeerreviewmaketitle

\section{Introduction}\label{intro}
CNNs were introduced in the 1990s by Lecun \cite{lecun1998gradient}
and gained enormous popularity in the past decade especially after the
demonstration of the significant success on Imagenet data by
Krizhevsky et al. \cite{krizhevsky2009learning}. At the heart of CNNs
success is its ability to learn a rich class of features from data
using a combination of convolutions and nonlinear operations such as
ReLU or softmax functions. The success of CNNs however is achieved at the
expense of a large number of parameters that need to be learned and
a computational burden in the training time. It is well
known now that a multi-layer perceptron can approximate any function
to the desired level of accuracy with a finite number of neurons in
the hidden layer. It is therefore natural to consider parameter efficiency as one of the network design goals to strive for in a deep netowrk.
The higher order Volterra series 
can capture a richer class of features and hence
significantly reduce the total number of parameters while maintaining comparable
or better classification accuracy relative to the baseline models. 

%In support of this
%conjecture, we empirically show that the total number of parameters in
%the proposed VolterraNet in comparison to say a spherical-CNN that is
%suited for data on spherical domains are in fact reduced
%significantly.

%% for a restricted class of functions but yet an expressive
%% and widely used class, namely the polynomial class, the number of
%% parameters in the network is smaller for a VHCNN than required for a
%% CNN to achieve this approximation tolerance.

In computer vision and medical imaging, many applications deal with
data domains that are non-Euclidean. For instance, the n-sphere ($n
\geq 2$), the manifold of symmetric positive definite matrices, the
Grassmannian, Stiefel manifold, flag manifolds etc. Most of these manifolds belong to
the class of (Riemannian) homogeneous spaces (manifolds).
Thus,
our goals here are to 
{\color{blue}
\begin{enumerate*}
\item Introduce a principled framework for defining CNNs on general homogeneous Riemannian manifolds.  
\item Introduce a novel higher order
convolution layer using Volterra theory \cite{volterra2005theory} on homogeneous Riemannian manifolds which provides significant parameter-efficiency improvements for non-Euclidean CNNs. 
\item Establish empirical evidence demonstrating the applicability of our homogeneous Riemannian manifold CNNs and the performance boost provided by the Volterra convolutions. 
\end{enumerate*}
}

Much of the recent work in this problem domain has focused on
generalizing CNNs to homogeneous spaces by exploiting the weight 
sharing that the symmetries of the underlying manifold allow. The 2-sphere is
a particularly important example. 
In the recent past, CNNs have been reported in literature
\cite{Worrall16,Cohen-ICLR18,Cohen-ICML17} which are designed to handle
data that are samples of functions defined on a 2-sphere and hence are equivariant to 3D rotations
which are members of the $\textsf{SO}(3)$ group. The spherical
convolution\footnote{As has been pointed out several times in the literature,
the convolution operation in CNNs is actually a
correlation and not a convolution. Hence, in this
paper, {\it we will use the term convolution and correlation
  interchangeably but always imply correlation}.} network presented in \cite{Cohen-ICLR18,esteves2018learning} is named
Spherical CNN. Recently, Kondor et al. in \cite{kondor2018clebsch}
proposed the Clebsch-Gordan net by replacing the repeated forward and
backward Fourier transform operations used in
\cite{Cohen-ICLR18}. They showed that by using the Clebsch-Gordan
transform as the source of nonlinearity, better performance can be
achieved by avoiding the repetitive forward and inverse Fourier transform
operations. In \cite{esteves2017polar}, authors present
polar transformer networks, which are equivariant to rotations and
scaling transformations. By combining them with the spatial transformer
\cite{jaderberg2015spatial}, they achieved the required equivariance to
translations as well. Recently, the equivariance of convolutions to
more general classes of group actions has been
reported in literature
\cite{kondor2018generalization}, and later in \cite{cohen2018general}. In
\cite{esteves2018learning}, Esteves et al. used the correlation
defined in \cite{driscoll1994computing} to propose an $\textsf{SO}(3)$
equivariant operator and in turn define a spherical convolution. In
this paper, we will define correlations and the Volterra series on a homogeneous manifold and 
show that the equivariance property holds for both.

Volterra kernels were first proposed in image classification
literature in \cite{kumar2009volterrafaces,kumar2012trainable}. In
\cite{kumar2012trainable}, authors learn the kernels in a data driven
fashion and formulate the learning problem as a generalized
eigenvalue problem. Volterra theory of nonlinear systems was applied
more than two decades ago to a single hidden layer feed-forward neural
network with a linear output layer and a fully dynamic recurrent
network in \cite{hakim1991volterra}.  Most recent use of Volterra
kernels in deep networks was reported in \cite{zoumpourlis2017non},
where, authors presented a single layer of  Volterra kernel based
convolutions followed by conventional CNN layers. They however did not explore equivariance properties of the network or consider non-Euclidean input domains. 

In this paper, we define a Volterra kernel to replace traditional convolution kernels. We present a novel generalization of the convolution group-equivariance
property to higher order convolutions expressed using Volterra theory of functional convolution on non-Euclidean domains, specifically, the Riemannian homogenous 
spaces \cite{helgason1962differential} referred to earlier.
%Examples of homogeneous spaces include Riemannian symmetric spaces such as the n-sphere ($\mathbf{S}^{n}$), projective space, the Grassmannian, manifold of$n\times n$ symmetric positive definite matrices denoted by $P_n$, and others. 
Most of these manifolds are commonly encountered in mathematical formulations of various computer vision tasks such as action recognition, covariance tracking etc., and in medical imaging for example, in diffusion magnetic resonance imaging (dMRI), elastography etc. By generalizing traditional CNNs in two possible ways, \begin{inparaenum} \item to cope with data domains that are non-Euclidean and \item to higher order convolutions expressed using Volterra series, we expect to extend the success of CNNs in yet unexplored ways. \end{inparaenum}

We begin with a significant extension of prior work by the authors of \cite{banerjee2019dmr},
where the authors defined a correlation operation for homogeneous manifolds. Specifically, our extension consists of a proof that not only is the correlation operation group equivariant, but additionally any linear group equivariant function can be written as a correlation on the manifold (Banerjee et al. \cite{banerjee2019dmr} only showed the first fact). 
We present experiments to demonstrate better performance of the proposed VolterraNet on spherical-MINST and the Shrec17 data with less number of parameters than previously shown in literature for the Spherical-CNN and the Clebsch-Gordan net. We then present a dilated convolution model based on the VolterraNet and demonstrate its efficacy in group testing on diffusion magnetic resonance data acquired from patients with movement disorders. The domain of this data is another example of a Riemannian homogeneous space.

In summary, our key contributions in this paper
are: \begin{inparaenum} 
\item A principled method for choice of basis in designing a deep network architecture on a Riemannian homogeneous manifold $\mathcal{M}$. 
\item A proof of a generalization of the classical linear shift invariance (in our terminology, equivariance) characterization theorem for correlation operations on Riemannian homogeneous manifolds. 
\item A novel generalization of convolution operations to higher order Volterra series on non-Euclidean domains specifically, Riemannian homogeneous manifolds which are often encountered both in computer vision and medical imaging applications. 
\item  A generalization of the classical non-linear shift invariance (in our terminology, equivariance) characterization theorem for  Volterra convolution operations on Riemannian homogeneous manifolds. 
\item Experiments on real data sets that are publicly
available such as the spherical-MNIST, atomic energy and
Shrec17. For these real data, we present comparisons to the
state-of-the-art methods. \item An extension of the VolterraNet to Dilated VolterraNet and demonstrate its efficiency via group testing on diffusion MRI brain scans from controls (normal subjects) and movement disorder patients. Further, ablation studies on VolterraNet to demonstrate the usefulness of the higher order convolution  operations. \end{inparaenum}
  
The rest of the paper is organized as follows: 
In section \ref{theory:correlation} we define the correlation operation on homogenous manifolds and prove a generalization of the Euclidean linear shift invariance (LSI) theorem for this correlation operation.
Then, in section \ref{theory:basis} we present a framework  for principled choice of basis in representing functions on a Riemmanian homogeneous manifold. In section \ref{theory:volterra}, we define the Volterra higher-order
convolution operation and prove a generalization of the non-linear shift invariance theorem for this Volterra operation. Following this, we present a detailed description of the proposed VolterraNet architecture in \ref{vol_architecture} and a description of the proposed dilated VolterraNet in \ref{dilated_theory}. 
Finally, section \ref{results} contains the experimental results and section \ref{conc} the conclusions.

\section{List of Notations}
We now summarize the list of notations that will be used throughout this paper.

{
    \centering
    \scalebox{0.85}{
    \begin{tabular}{|l|l|}
    \hline
        $\mathcal{M}$ & Riemannian homogeneous space (manifold) \\
        $G$ & a group \\
        $\textsf{SO}(n)$ &  n-dimensional special orthogonal group \\ 
        & of matrices \\
        $I(\mathcal{M})$ & Isometry group admitted by $\mathcal{M}$ \\
        $L^2(\mathcal{M}, \mathbf{R})$ & Space of real-valued square integrable \\
        & functions on $\mathcal{M}$\\
        $L^2(G, \mathbf{R})$ & Space of real-valued square integrable \\
        & functions on $G$\\
        $\omega^{\mathcal{M}}$ & Volume form of $\mathcal{M}$ \\
        $\mu_G$ & Haar measure of $G$ \\
        $g\cdot x/ L_g(x)$ & Action of $g\in G$ on $x\in \mathcal{M}$ \\
        $gh/ L_g(h)$ & Action of $g\in G$ on $h\in G$ \\
        $g\cdot f/ L^*_{g^{-1}}(f)$ & Action of $g\in G$ on $f:\mathcal{M}\rightarrow \mathbf{R}$ \\
        & and is given by $x \mapsto f(g^{-1}\cdot x)$\\
        $\mathbf{S}^n$ & $n$-sphere \\
        $\mathbf{R}^+$  & Space of positive reals \\
        $P_3$ & Space of $3\times 3$ symmetric \\ & positive-definite matrices \\
        $\textsf{GL}(n)$ & General linear group of $n\times n$ matrices \\
        $\textsf{O}(n)$ & Space of $n\times n$ orthogonal matrices \\
        $\mathbf{R}\setminus \left\{0\right\}$ & Space of reals without the origin \\
        $\text{Stab}(x)$ & Stabilizer of an element $x\in \mathcal{M}$ \\
        $g^{\mathcal{M}}$ & Riemannian metric on the manifold $\mathcal{M}$ \\
        $logm$ & Matrix log operation\\
      \hline
    \end{tabular}
    }
    %\caption{}
    }
    %\label{tab:my_label}
%\end{table}

\section{Correlation on Riemannian homogeneous spaces}
\label{theory:1}
In this section, we define a correlation operation which generalizes the
Euclidean convolution layer to arbitrary Riemannian homogeneous spaces. Further, we
prove a generalization to Riemannian homogeneous spaces of the linear shift-invariant system (LSI) characterization of Euclidean convolutions.  
Similar theorems were first proved in \cite{kondor2018generalization} and later in \cite{cohen2018general}.
This result is not meant to be novel but to motivate the analogous result for higher order convolutions on Riemannian homogeneous spaces that we prove subsequently.

\subsection{Background}
\label{theory:back}
We will briefly review the differential geometry of Riemannian homogeneous spaces from 
an informal perspective. Formal definitions will be deferred to the appendix for conceptual clarity. 

\vspace{0.1cm}
As mentioned earlier, Riemannian homogeneous spaces are Riemannian manifolds which `look' the
same locally at each point with respect to some symmetry group $G$, meaning that the action of
$G$ on $\mathcal{M}$ is transitive.  We will specifically consider Riemannian manifolds with a 
transitive action of the isometry group $I(\mathcal{M})$.  For the rest of the paper, we use $G = I(\mathcal{M})$ unless mentioned otherwise. For example, the $2$-sphere is a 
homogeneous space with $G = \textsf{SO}(3)$. An
important fact about homogeneous spaces is that they can be identified as a quotient space. 
In general, if $\mathcal{M}$ is a homogeneous space with group $G$ acting on it, and $H_x$ is some stabilizer (see definition in Appendix A) of
of a point $x \in \mathcal{M}$ then $\mathcal{M} \simeq G/H_x$. Returning to the $2$-sphere example, if we
take $H$ to be the stabilizer of the north pole, a subgroup of $\textsf{SO}(3)$ isomorphic to $\textsf{SO}(2)$, then
$G/H \simeq \textsf{SO}(3)/\textsf{SO}(2) \simeq \mathbf{S}^2$. 
For a detailed exposition on these concepts, we refer the reader to 
\cite{helgason1962differential}. 

\subsubsection{Assumptions}

For the remainder of this paper we assume $\mathcal{M}$ to be a Riemannian homogeneous space
admitting a transitive action of the group $G$, which we call the 
symmetries of $\mathcal{M}$. We also assume that $G$ is a locally compact topological group. 
Further we assume that any function $f : \mathcal{M} \to \mathbf{R}$ is square
integrable, i.e. $
    \int_\mathcal{M} |f(x)|^2 \omega^{\mathcal{M}}(x) < \infty$, 
where $\omega^{\mathcal{M}}$ is a suitable volume form on $\mathcal{M}$.  As mentioned before, we denote the space of square integrable functions on $\mathcal{M}$
by $L^2(\mathcal{M}, \mathbf{R})$. 
\subsection{Euclidean LSI Theorem}

We begin this section by recalling the Euclidean Linear Shift Invariant (LSI) theorem.

\begin{definition}
    Let $F : U \to V$ be a bounded linear operator between spaces $U$ and $V$ consisting of functions 
    $\mathbf{R}^n \to \mathbf{R}$. For $f \in U \cup V$ and $x \in \mathbf{R}^n$ we define $\tau_x(f)(z) = f(z-x)$.
    The set $\{\tau_x\}_{x \in \mathbf{R}^n}$ forms a group under composition. We
    say $F$ is translation equivariant (i.e. shift invariant in the traditional literature) if
    \[
        \tau_x(F(g)) = F(\tau_x(g))
    \]
    for all $g \in U$, $x \in \mathbf{R}^n$.
\end{definition}

\begin{theorem}%\footnote{The exact statement requires some assumptions on $U$, $V$ and $F$}
    Let $w : \mathbf{R}^n \to \mathbf{R}$ be a weight kernel, then the operator given as $G_w : U \to V$ is defined by $G_w(f) = f \star w$, 
    where $\star$ is the Euclidean convolution operation which 
    is a bounded, linear, and translation equivariant operator. Further, if $F$ is any bounded linear translation equivariant operator, then 
    there exists $w : \mathbf{R}^n \to \mathbf{R}$ such that $F = G_w$, i.e. $
        F(f) = f \star w 
    $, 
    for all $f \in U$.
\end{theorem}

Thus, Euclidean convolutions with a weight kernel have an interesting and powerful characterization as linear shift invariant operators. Next we 
show that the correlation operation on any Riemannian homogeneous manifolds satisfy a generalization of the aforementioned LSI theorem. 

\subsection{Generalizing Convolutions and the LSI Theorem to Riemannian homogeneous spaces}
\label{theory:correlation}

We begin by defining the correlation operation for arbitrary homogeneous Riemannian manifolds. Some equivalent definitions have been made several times in
the literature, first for specific manifolds as in \cite{Cohen-ICLR18}, \cite{esteves2018learning}, then in more generality such as in
\cite{kondor2018generalization} and later in \cite{cohen2018general}. We then state a generalization of the LSI theorem for this correlation operation, which we call the Linear Group Equivariant (LGE) theorem. Note that similar theorems were first proved in \cite{kondor2018generalization} and later in \cite{cohen2018general}. We present this theorem not as a novel result, but as motivation for a non-linear version of the theorem which we will prove in section \ref{theory}. Regardless, we present a much simpler proof (compared to \cite{kondor2018generalization}, \cite{cohen2018general}) of the result in the appendix.

\begin{definition} [{\bf Correlation}]
\label{theory:def5} 
The correlation between $f : \mathcal{M} \to \mathbf{R}$ and $w : \mathcal{M} \to \mathbf{R}$ is
given by, $\left(f \star w\right) : G \rightarrow \mathbf{R}$ defined
as follows:
\begin{align}
\left(f \star w\right)\left(g\right) := \int_{\mathcal{M}} f(x)\left(g\cdot w\right)(x) \omega^{\mathcal{M}}(x)
\label{ManCorr}
\end{align}

The correlation between $f : G \to \mathbf{R}$ and $w : G \to \mathbf{R}$ is
given by, $\left(f \star w\right) : G \rightarrow \mathbf{R}$ defined
as follows:
\begin{align}
\left(f \star w\right)\left(g\right) := \int_{G} f(h) \left(g\cdot w \right)(h) \mu_G(h)
\label{GroupCorr}
\end{align}
where $\mu_G$ is the Haar measure on $G$ (which is guaranteed to exist based on our assumption that $G$ is a locally compact topological group). 
Please see the discussion at the end of this subsection for details. 
\end{definition}

Equation \ref{ManCorr} is described in words as follows: the weight kernel $w$ is ``shifted'' using the action of the symmetry group, and the point-wise product of the shifted weight kernel and the function $f$ is integrated over the manifold. A similar interpretation can be given to the correlation on groups in Eq.  \ref{GroupCorr}.  This generalizes the work 
on the $2$-sphere presented in \cite{Cohen-ICLR18}, \cite{esteves2018learning} for an arbitrary Riemannian homogeneous 
space $\mathcal{M}$. 

We show that this correlation operation is equivariant to the isometry group $G$ of the underlying homogeneous space 
$\mathcal{M}$. In order to state this theorem, we first formally define equivariance. 

\begin{definition} [{\bf Equivariance}]
\label{theory:def4}
Let $X$ and $Y$ be sets and $G$ be a group acting on $X$ and $Y$ (in literature these sets are termed as $G$ sets \cite{dummit2004abstract}). Then, $F: X \rightarrow Y$ is said to be
{\it equivariant} to the action of $G$ if
\begin{align}
F(g\cdot x) = g\cdot F(x)
\end{align}
for all $g \in G$ and all $x \in X$.
\end{definition}

We are now ready to state the following theorems: 
\begin{theorem} 
\label{theory:corr_eq}
Let $S$ and $U$ be $L^2(\mathcal{M}, \mathbf{R})$ or $L^2(G, \mathbf{R})$ and 
$F: S \rightarrow U$ be a function given by $f \mapsto \left(f \star w\right)$. Then $F$ is equivariant
with respect to the pullback action of $G$, i.e.
\[
    (\phi \cdot f) \star w = \phi \cdot (f \star w) 
\]
for $\phi \in G$ a symmetry of $\mathcal{M}$ where
\[
    \phi \cdot h \coloneqq h \circ \phi^{-1}
\]
for any $h : \mathcal{M} \to \mathbf{R}$ square-integrable. 
\end{theorem}
\begin{proof}
    See appendix \hyperref[appendix:pfs]{B}.
\end{proof}

This constitutes the forward direction of the LGE theorem. Now we show the converse statement, namely, every linear group equivariant function is a correlation. 

\begin{theorem}
    Let $S$ and $U$ be $L^2(\mathcal{M}, \mathbf{R})$ or $L^2(G, \mathbf{R})$ and $F: S \rightarrow U$ 
    be a linear equivariant function with respect to the pullback action of $I(\mathcal{M})$. Then,
$\exists w \in S$ such that, $\left(F(f)\right)(g) = \left(f \star
w\right)(g)$, for all $f \in S$ and $g \in G$.
\end{theorem}
\begin{proof}
    See appendix \hyperref[appendix:pfs]{B}.
\end{proof}

Together with these two theorems, we can generalize the LSI theorem to homogeneous spaces using the correlation defined in Def. \ref{theory:def5} . 
\newline
\newline
\fbox{\begin{minipage}{25em}
\textsc{A note on volume forms / measures}
\newline
\newline
In Def. \ref{theory:def5}, we specify the Haar measure for integration of a function on $G$. 
If $f$ and $w$ are functions on $G$, then the Haar measure $\mu_G$ has several desirable properties. For example, the Haar measure 
is invariant to ``translations'', i.e. if $S \subset G$ is measurable then $\mu_G(S) = \mu_G(gS)$ for any $g \in G$.
Further, using the Haar measure provides a convolution theorem which makes correlation a simple multiplication under the generalized Fourier transform for groups. 
This particular property is vital for efficient implementations. 

Note that on the other hand, we do not specify a specific volume form $\omega^\mathcal{M}$ for integration of function on $\mathcal{M}$ in definition \ref{theory:def5}. 
In many cases, the Haar measure on $G$ will induce a $G$-invariant volume form on $\mathcal{M} \simeq G/H$, but stating the exact conditions for this to be possible requires
some work. Instead, we define the correlation using an arbitrary volume form. In the next section we will give a construction which induces such a $G$-invariant volume form on $\mathcal{M}$.
\end{minipage}}

\subsection{Basis functions for $L^2$-functions on homogeneous spaces}
\label{theory:basis}
Our goal in this section is to induce a natural basis on $L^2 \left( \mathcal{M}, \mathbf{R} \right)$ 
from the canonical basis on $L^2 \left(G, \mathbf{R} \right)$ where
$G$ is the group acting on the homogeneous manifold $\mathcal{M}$. The basis on $G$ consists of matrix elements of irreducible 
unitary representations, which provides a Fourier transform on $G$ (for more details reader is referred to \cite{hewitt2012abstract}). 
We show that this construction matches the commonly used basis for specific manifolds, e.g. the spherical harmonics and Wigner-D functions in \cite{Cohen-ICLR18}.
This construction can be used to induce basis on arbitrary Riemannian homogeneous spaces.

\subsubsection{Basis on $L^2\left(\mathcal{M}, \mathbf{R}\right)$ induced from $L^2\left(G, \mathbf{R}\right)$}

 To induce a basis on $L^2 \left(\mathcal{M}, \mathbf{R} \right)$, we use the principal fiber bundle structure of the homogeneous manifold $\mathcal{M}$. 
A fiber bundle is a space that locally looks like a product space. It is expressed as a base space $B$ with the fibers making up a fiber space $F$, and their union being the 
total space denoted by $E$. There is a projection map $\pi : E \to B$ mapping fibers to their "base point" on $B$. A principal fiber $G$-bundle is a fiber bundle with a continuous (right) action of
a group $G$, such that the action of $G$ is free, transitive and preserves the fibers. For more details on fiber bundle theory see \cite{michor2008topics}. 

As mentioned in \ref{theory:back}, $\mathcal{M}$ can be identified with $G/H$ for $G$ the group action on $\mathcal{M}$ and $H$ the stabilizer of a point $x \in \mathcal{M}$,
usually called the "origin". It is well known that this identification induces a principal fiber $G$-bundle structure on $\mathcal{M}$ via the projection map. 

\begin{wrapfigure}{r}{4.4cm}
    \vspace{-2em}
    \begin{center}
      \includegraphics[scale=0.25]{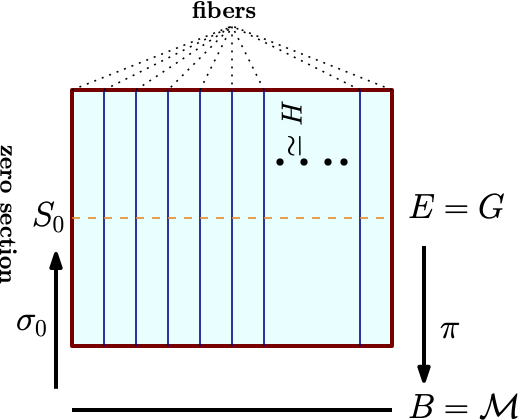}
\caption{\footnotesize Fiber bundle $(B, E, \pi)$.}\label{figbasis}
      \end{center}
      \vspace{-1.5em}
\end{wrapfigure}

\begin{proposition} \cite{helgason1962differential} % please cite the appropriate reference.
\label{theory:prop3.5}
The homogeneous space, $\mathcal{M}$ identified as $G/H$ together with
the projection map $\pi: G \rightarrow G/H$ is a principal bundle with
$H$ as the fiber. Furthermore there exists a diffeomorphism $\psi:G/H \rightarrow \mathcal{M}$ given by $gH \mapsto g\cdot o$, where $o$ is the ``origin'' of $\mathcal{M}$.
\end{proposition}

Moreover, a section is a continuous right inverse of $\pi$, which is denoted by $\sigma: B\rightarrow E$. In literature \cite{helgason1962differential}, a zero section (denoted by $S\subset G$) is the section containing the identity element of $H$. Let $\sigma_0: S\rightarrow \mathcal{M}$ be a  diffeomorphism. Given $\left\{v_{\alpha}:G \rightarrow \mathbf{R}\right\}$ be the set of basis of $L^2\left(G, \mathbf{R}\right)$. Then, we can get the induced basis on $L^2\left(\mathcal{M}, \mathbf{R}\right)$ as $\left\{\widetilde{v}_{\alpha} = v_{\alpha}\circ \sigma_0^{-1}\right\}$. A schematic of an example fiber bundle is shown in Fig. \ref{figbasis}.

{\bf Example:} Consider the example of $\mathcal{M}$ $=\mathbf{S}^2$, where $G =$ $\textsf{SO}(3)$, $H = \textsf{SO}(2)$. A choice of basis on $L^2\left(G, \mathbf{R}\right)$ is Wigner D-functions denoted by $\left\{D_{l,m}^j | j \in \{0, 1, \cdots, \infty\}, -j\leq l, m\leq j\right\}$. Let $(\alpha, \beta, \gamma)$ be the parametrization of $\textsf{SO}(3)$ and the zero section ($S$) be denoted by $\{\alpha, \beta, 0\}$.  Then, $\left\{D_{l,0}^j\right\}$ are the choice of basis on $S\subset G$, which gives the induced basis on $\mathbf{S}^2$ as $\widetilde{D}_{l}^j(\theta, \phi) = \sqrt{\frac{2l+1}{4\pi}} D_{l,0}^j(\phi, \theta, 0)$. Further, observe that $\left\{\widetilde{D}_{l}^j\right\}$ are the spherical harmonics basis.

\section{Higher order correlation on Riemannian homogeneous spaces}
\label{theory}

In this section, we define a higher order correlation operator on
Riemannian homogeneous spaces using the Volterra series, state a theorem demonstrating it's
symmetry equivariance and show how to compute it efficiently using 
first order correlation operations. Further, we  prove that the set of functions which can be written as sums of products of linear operators and are $G$-equivariant can be expressed as a Volterra series.
This partially generalizes the non-linear shift equivariance characterization of Volterra expansions in
Euclidean space. 

\subsection{Volterra Series on Homogenous Spaces}
\label{theory:volterra}
We now generalize the Volterra Series to Riemannian homogeneous spaces.

\begin{definition}[Volterra series expansion]
\label{theory:def6} 
We define the Volterra expansion of a function $f:\mathcal{M} \rightarrow \mathbf{R}$ or $f:G \rightarrow \mathbf{R}$ by $F(f) = \sum_{n=1}^{\infty} \left(f \star_{n} w_n\right)$. 
If $f : \mathcal{M} \to \mathbf{R}$ and $w_n : \left(\mathcal{M}\right)^{\oplus n} \rightarrow \mathbf{R}$ then $\left(f \star_{n} w_n\right) : G \rightarrow \mathbf{R}$ is defined as,
\begin{align*}
\left(f \star_{n} w_n\right)(g) &:= \int_{\mathcal{M}}\cdots  \int_{\mathcal{M}} f(x_1) \cdots f(x_n) \left( g \cdot w_n\right) \\ &(x_1, \cdots, x_n) \omega^{\mathcal{M}}(x_1) \cdots \omega^{\mathcal{M}}(x_n)
\end{align*}
If instead $f : G \to \mathbf{R}$ and $w_n : \left(G\right)^{\oplus n} \rightarrow \mathbf{R}$ then $\left(f \star_{n} w_n\right) : G \rightarrow \mathbf{R}$ is defined as,
\begin{align*}
\left(f \star_{n} w_n\right)(g) &:= \int_{G}\cdots  \int_{G} f(h_1) \cdots f(h_n) \left( g \cdot w_n  \right) \\ &(h_1, \cdots, h_n) \mu_G(h_1) \cdots \mu_G(h_n)
\end{align*}
where $\mu_G$ is the Haar measure on $G$ (which again, is guaranteed to exist based on our assumption that $G$ is a locally compact topological group). 
\end{definition}

One can easily see that, Definition \ref{theory:def5} is a special
case of definition \ref{theory:def6} when $n=1$. When $n>1$, we will
call it the $n^{th}$ order Volterra expansion. Higher order terms of the Volterra 
expansion express polynomial relationships between function values. An 
illustration of the second order Volterra kernel is provided in Figure \ref{theory:fig_sec}. { As
we can see, the second order Volterra kernel has a regular correlation weight kernel at each location on the manifold
$\mathcal{M}$. The results of applying these weight kernels get multiplied together to get the output of $f \star_2 w_2$. }
A biological motivation is provided in \cite{zoumpourlis2017non} for the (Euclidean) Volterra series. Now, we prove that $F$ as defined in Definition \ref{theory:def6} is equivariant to the
symmetry group actions admitted by a homogeneous space.

\begin{theorem}
Let $S$ and $U$ be $L^2(\mathcal{M}, \mathbf{R})$ or $L^2(G, \mathbf{R})$ and $F: S \rightarrow U$ be a function given by $f \mapsto
\sum_{n=1}^{\infty} \left(f \star_{n} w_n\right)$. Then, $F$ is
equivariant.
\label{theory:thm3}
\end{theorem}
\vspace*{-1em}
\begin{proof}
 Observe that the sum of
equivariant operators is equivariant.  Hence, we only need to check
that $f \star_{n} w_n$ is equivariant for all $n$. Let $g, h \in G$,
let $n \in \mathbf{N}$. Then,
{\color{blue}
\begin{align*}
\left(g.f \star_{n} w_n\right)\left(h\right) &= \left(L^*_{g^{-1}}f \star_{n} w_n\right)\left(h\right) \\
&= \int_{\mathcal{M}}\cdots  \int_{\mathcal{M}} L^*_{g^{-1}}f(x_1) \cdots L^*_{g^{-1}} f(x_n) \\ &\left(L^*_{h^{-1}}w_n\right)(x_1, \cdots, x_n) \omega^{\mathcal{M}}(x_1) \cdots \omega^{\mathcal{M}}(x_n) \\
&= \int_{\mathcal{M}}\cdots  \int_{\mathcal{M}} f(y_1)\cdots f(y_n) w_n\left((h^{-1}g)\cdot y_1 \right. \\&\left. ,\cdots, (h^{-1}g)\cdot y_n\right) \omega^{\mathcal{M}}(g\cdot y_1) \cdots \omega^{\mathcal{M}}(g\cdot y_n) \\
&= \int_{\mathcal{M}}\cdots  \int_{\mathcal{M}} f(y_1)\cdots f(y_n) w_n\left((h^{-1}g)\cdot y_1 \right. \\&\left. ,\cdots, (h^{-1}g)\cdot y_n\right) \omega^{\mathcal{M}}(y_1) \cdots \omega^{\mathcal{M}}(y_n) \\
&= \left(f \star_{n} w_n\right)(g^{-1}h) \\
&= L^*_{g^{-1}} \left(f \star_{n} w_n\right) \left(h\right) \\
&= \left(g\cdot \left(f\star_{n} w_n\right)\right)\left(h\right)
\end{align*}
}
{ Here, $(L_g^* f)(h) = f(g^{-1}h)$ (see appendix for details)}, since, $g, h \in G$ and $n$ are arbitrary $F$ is equivariant.
\end{proof}

In the other direction, we also show that for the aforementioned set of functions, every $G$-equivariant
function can be written as a Volterra series. 

\begin{theorem}
{\color{blue} Let $S$ and $U$ be $L^2(\mathcal{M}, \mathbf{R})$ or $L^2(G, \mathbf{R})$ and $F: S \rightarrow U$ be a non-linear $G$-equivariant function which can be written as
$F = \sum_{i \in I} F_i$, where each $F_i$ is a product of two linear functions, i.e., $F_{i} = F_{i,1}F_{i,2}$. Then,
$\exists \left\{w_{i}\right\}_{i \in I} \subset S$ such that, $\left(F(f)\right)(g) = \sum_{i\in I}\left(f \star_2
w_{i}\right)(g)$, for all $f \in S$ and $g \in G$.}
\end{theorem}
\begin{proof}
    It suffices to show that for each term $F_i = F_{i,1} F_{i,2}$ (for $F_{i,k}$ a linear $G$-equivariant function) 
there exists $w_i$ such that $F_i = f \star_2 w_i$. If $w_i(x,y) = w_{i,1}(x) w_{i,2}(y)$
(i.e. $w_i$ is separable), then $f \star_2 w_i = (f \star w_{i,1}) (f \star w_{i,2})$. 
But by the previous theorem, there exists $w_{i,k}$ such that $F_{i,k} = f \star w_{i,k}$,
completing the proof. 
\end{proof}

These results partially generalize the well know non-linear shift equivariance characterization of Volterra expansions in
Euclidean space and justifies the use of the Volterra series as a higher-order generalization of the
correlation operation Definition \ref{theory:def5}. 

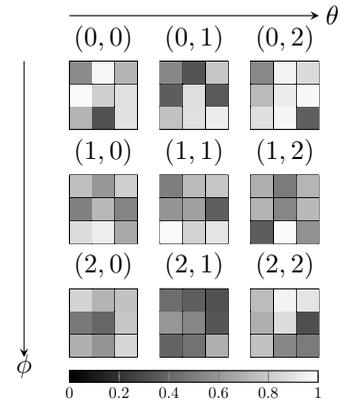
\begin{figure}
    \centering
    \begin{tikzpicture}[scale=0.6]
        %grids
        \draw[step=0.5cm,color=black] (0,0) grid (1.5,1.5);
        \draw[step=0.5cm,color=black] (1.99,0) grid (3.5,1.5);
        \draw[step=0.5cm,color=black] (3.99,0) grid (5.5,1.5);
        
        \draw[step=0.5cm,color=black] (0,1.99+0.5) grid (1.5,3.5+0.5);
        \draw[step=0.5cm,color=black] (1.99,1.99+0.5) grid (3.5,3.5+0.5);
        \draw[step=0.5cm,color=black] (3.99,1.99+0.5) grid (5.5,3.5+0.5);
        
        \draw[step=0.5cm,color=black] (0,3.99+1) grid (1.5,5.5+1);
        \draw[step=0.5cm,color=black] (1.99,3.99+1) grid (3.5,5.5+1);
        \draw[step=0.5cm,color=black] (3.99,3.99+1) grid (5.5,5.5+1);
        
        \foreach \xx in {0.0, 2.0, 4.0}{
            \pgfmathsetmacro \xs{\xx}
            \pgfmathsetmacro \xm{add(\xx,0.5)}
            \pgfmathsetmacro \xe{\xx+1}
            
            \foreach \yy in{0.0, 2.0+0.5, 4.0+1}{
                \pgfmathsetmacro \ys{\yy}
                \pgfmathsetmacro \ym{\yy+0.5}
                \pgfmathsetmacro \ye{\yy+1}
                 \foreach \y in {\ys,\ym,\ye} {
                    \foreach \x in {\xs,\xm,\xe} {
                        \pgfmathparse{70*rnd+30}
                        \edef\tmp{\pgfmathresult}
                        \fill [white!\tmp!black] (\x+0.01,\y+0.01) rectangle ++(0.48,0.48);
                }
            }
            }
        }
        
        \node () at (0.75,6+1) {$(0,0)$};
        \node () at (2.75,6+1) {$(0,1)$};
        \node () at (4.75,6+1) {$(0,2)$};
        
        \node () at (0.75,4.5) {$(1,0)$};
        \node () at (2.75,4.5) {$(1,1)$};
        \node () at (4.75,4.5) {$(1,2)$};
        
        \node () at (0.75,2) {$(2,0)$};
        \node () at (2.75,2) {$(2,1)$};
        \node () at (4.75,2) {$(2,2)$};
        
        \node () at (5.8, 7.5) {$\theta$};
        \node () at (-1, -0.2) {$\phi$};

        \draw [->,>=stealth,blue] (0,7.5) -- (5.5,7.5);
        \draw [->,>=stealth,blue] (-1,6.5) -- (-1,0);
        
        \begin{axis}[
        %view={0}{90},
        hide axis,
        colormap/blackwhite, 
        colorbar horizontal,
        colorbar style={
            width=0.8*\pgfkeysvalueof{/pgfplots/parent axis width},
        },
        colorbar/width=2.5mm
    ]
    \end{axis}

    \end{tikzpicture}
    \caption{ Visualization (in the spirit of \cite{zoumpourlis2017non}) of second-order term $w_2 : \mathcal{M}^2 \to \mathbf{R}$ of a Volterra kernel (here on a $2$-manifold parametrized by $(\theta, \phi)$).  The coordinates above each grid represent the first entry $w_2(\mathbf{x},\cdot)$, and within each grid the gray-scale value represents the weight of the associated kernel $w_2(\mathbf{x},\mathbf{y})$. } 
    \label{theory:fig_sec}
\end{figure}

\subsection{Efficient Computation of the Second-Order Volterra Kernel}

The Volterra series presented in the previous definition is significantly more expressive 
than the correlation operation defined in Definition \ref{theory:def5} since it captures higher order 
relationships between inputs, but it requires the computation of iterated integrals and 
does not have an efficient GPU implementation. { Note that for separable second order kernel $w_2$,  $\left(f \star_{2} w_2\right)(g)$ can be factored as $\left(\left(f\star
\tilde{w}_2\right)\left(g\right)\right) \left(\left(f\star
\bar{w}_2\right)\left(g\right)\right)$}. Thus, we can compute
the second order Volterra series with separable kernel as a product of traditional correlation
operations. 
In general, we can
use a convex combination of first order and second order terms of the Volterra series to
define second order Volterra network. 

\begin{figure}[!ht]
    \vspace{-0em}
    \begin{center}
      \includegraphics[scale=0.35]{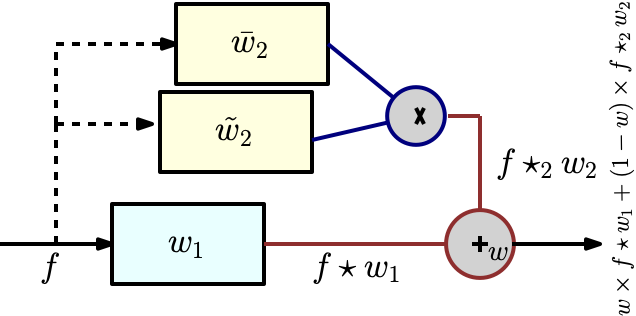}
\caption{\footnotesize {Second order Volterra correlation operator with first order kernel $w_1$ and separable second order kernel $w_2$.}}\label{fig0}
      \end{center}
      \vspace{-1.5em}
\end{figure}

A schematic diagram for the second
order Volterra correlation operator is shown in Fig. \ref{fig0}. {\color{blue} This representation of a second order kernel using product of two separable kernels is analogous to tensor product approximation of a function and can be shown to achieve approximation error of an  arbitrary precision \cite{hackbusch2007tensor}. The separability assumption on the kernels leads to efficient computation which is especially valuable in the network setting where these operations are performed numerous times.} %Moreover, some of the separable kernels like Gaussian, Sobel and Prewitt are popularly used in image/signal processing literature \cite{rosenfeld1976digital}.} 
%\vspace{-2em}
\section{Architecture}
\label{vol_architecture}
We now present the basic modules for implementing our correlation and higher-order
Volterra operations as layers in a deep network.

\subsection{Correlation on homogeneous spaces}
Using definition \ref{theory:def5} we can define:
\vspace*{0.1cm} 

{\bf Correlation on $ \mathcal{M}$ - $\textsf{Corr}^{\mathcal{M}}(f,w)$}: Let $f \in
L^2\left(\mathcal{M}, \mathbf{R}\right)$ be the input function and $w
\in L^2\left(\mathcal{M}, \mathbf{R}\right)$ be the mask. Then, using
definition \ref{theory:def5}, $\textsf{Corr}^{\mathcal{M}}(f,w)$ is defined as
$\left(f \star w\right): G \rightarrow \mathbf{R}$. We have shown in Theorem
\ref{theory:corr_eq}, that $\textsf{Corr}^{\mathcal{M}}(f,w)$ is
equivariant to the action of $G$. Hence, we can use $\textsf{Corr}^{G}$ 
layer as the next layer.
\vspace*{0.1cm}

{\bf Correlation on $G$ - $\textsf{Corr}^{G}(f,w)$}: Let $\widetilde{f} \in
L^2\left(G, \mathbf{R}\right)$ be the input function and $w \in
L^2\left(G, \mathbf{R}\right)$ be the mask. Then analogous to
$\textsf{Corr}^{\mathcal{M}}$, we can define $\textsf{Corr}^{G}(f,w)$ as $\left(\widetilde{f}
\star w\right): G \rightarrow \mathbf{R}$ using definition
\ref{theory:def5}. We have used
Theorem \ref{theory:corr_eq} to show that $\textsf{Corr}^{G}(f, w)$ 
is equivariant to the action of $G$. {\it Since this is an operation
  equivariant to $G$, we can cascade $\textsf{Corr}^{G}$}.

\subsection{Volterra on homogeneous spaces}

We can see that because the basic
architecture of second order Volterra series consists of the following
modules:
\vspace*{0.1cm} 

{\bf Second order Volterra on $\mathcal{M}$ - $\textsf{Corr}_2^{\mathcal{M}}(f,
  w_1, w_2)$:} Let $f \in L^2\left(\mathcal{M}, \mathbf{R}\right)$ be the
input function and $w_1 : \mathcal{M} \rightarrow \mathbf{R}$ and $w_2 : \left(\mathcal{M}\right)^{\oplus 2} \rightarrow \mathbf{R}$ be
the kernels. Then, $\textsf{Corr}_2^{\mathcal{M}}(f,
  w_1, w_2) := \sum_{j=1}^2 \left(f \star_j
w_j\right): G \rightarrow \mathbf{R}$. We have shown in Theorem
\ref{theory:thm3}, that $\textsf{Corr}_2^{\mathcal{M}}(f,
  w_1, w_2)$ is
equivariant to the action of $G$. Hence, we can use $\textsf{Corr}_2^{G}(f,
  w_1, w_2)$ layer as the next layer.
%\vspace*{0.1cm}

{\bf Second order Volterra on $G$ - $\textsf{Corr}_2^{G}(f, w_1,
  w_2)$:} Let $f \in L^2\left(G, \mathbf{R}\right)$ be the input
function and $w_1 : G \rightarrow \mathbf{R}$ and $w_2 :
\left(G\right)^{\oplus 2} \rightarrow \mathbf{R}$ be the
kernels. Then, $\textsf{Corr}_2^{G}(f, w_1, w_2) := \sum_{j=1}^2
\left(f \star_j w_j\right): G \rightarrow \mathbf{R}$. We have used
Theorem \ref{theory:thm3} to show that $\textsf{Corr}_2^{G}(f, w_1,
w_2)$ is equivariant to the action of $G$. {\it Since this is an operation
  equivariant to $G$, we can cascade $\textsf{Corr}_2^{G}(f, w_1,
  w_2)$}.
%\vspace*{0.2cm}
\subsection{Other Layers}
{\bf Activation function:} Since the outputs of all the above layers are functions from $G$ to $\mathbf{R}$, we
will use the standard activation operation on $\mathbf{R}$.

\vspace*{0.1cm}
{\bf Invariant last layer:} { As both layers, $\textsf{Corr}_2^{\mathcal{M}}$ and $\textsf{Corr}_2^{G}$ are equivariant to the action of $G$, so are the cascaded layers. Since, if the input signal is transformed by a group element $g\in G$, so is the output of $\textsf{Corr}_2^{\mathcal{M}}$ as this layer is equivariant. Thus the output of $\textsf{Corr}_2^{\mathcal{M}}$ is transformed by the same group element $g$. Hence, the input of $\textsf{Corr}_2^{G}$ is transformed by $g$ and due to the equivariance so is the output of $\textsf{Corr}_2^{G}$. This justifies that the cascaded layers are equivariant to the action of $G$.}
Hence, after the cascaded correlation layers, the
output $\widetilde{f} \in L^2\left(G,
\mathbf{R}\right)$ lies on a $G$ set. Similar to the Euclidean CNN, we want the last
layer to be $G$ invariant. Hence, we will integrate $\widetilde{f}$ on
the domain $G$ and return a scalar. Note that in the experiment, we
learn multiple channels analogous to the Euclidean CNN, where in each
channel, we learn a $G$ equivariant $\widetilde{f} \in L^2\left(G,
\mathbf{R}\right)$. Thus, after the integration, we have $c$ scalars,
where $c$ is the number of channels, which will be input to the
softmax fully connected layer similar to the Euclidean CNN. { We
abbreviate this last layer as {\bf iL} (invariant layer).}

A schematic diagram of our proposed VolterraNet is shown in Fig. \ref{fig1}.

\begin{figure}[!ht]
        \centering
                \includegraphics[scale=0.25]{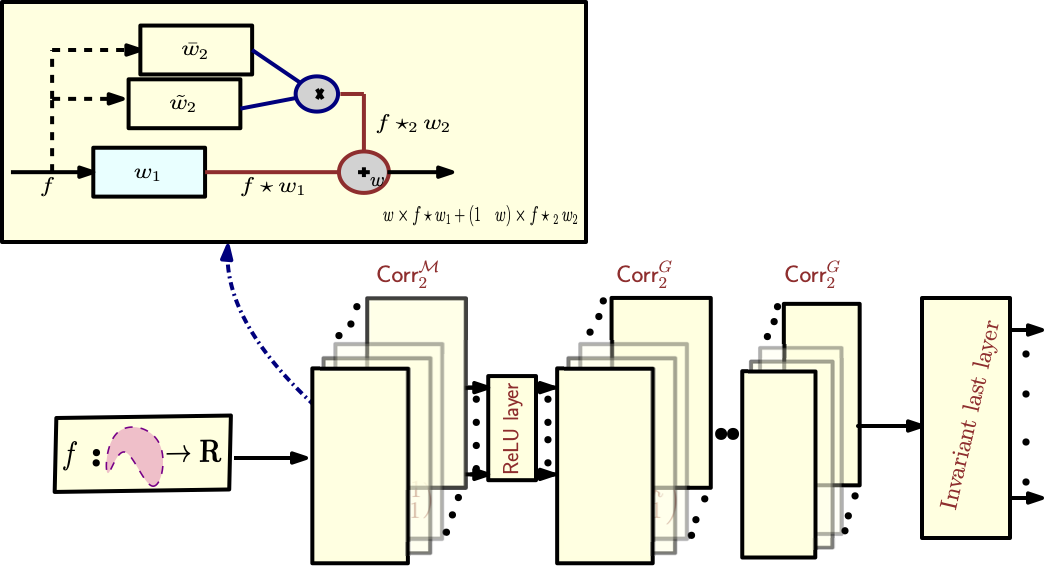}
               \caption{\footnotesize Schematic diagram of a second order Volterranet}\label{fig1}
\end{figure}

\section{Dilated VolterraNet}\label{dilated_theory}

In this section, we propose a dilated VolterraNet framework which is suitable for sequential data. {Sequential data here refers to a sequence of data points (signal measurements at voxels) along the neuronal fiber tracts that are extracted from diffusion MRI data sets. Neuronal fiber tracts in certain regions of the brain are disrupted by movement disorders such as Parkinsons disease. The sensory motor area tract (pathway) in the brain is one such neuronal pathway where the disease caused changes are expected to be observed. By treating this pathway as a sequence of points where diffusion sensitized MR signal is acquired, we propose to apply the dilated VolterraNet (described below) to analyze this data.} 

{It is known that the sequential data should involve recurrent structure \cite{elman1990finding}, but as pointed out in \cite{bai2018convolutional}, convolutional architectures often outperform recurrent models in sequential data analysis. Furthermore, recurrent models are computationally more expensive than the convolutional models. But, note that in order to mimic the infinite memory capabilities of a recurrent model, one needs to increase the receptive field by using the dilated convolutions.}
We will first recap the definition of Euclidean dilated convolution \cite{bai2018convolutional}  and then describe the proposed dilated VolterraNet. 

\subsection{Euclidean Dilated Convolution:}\label{dilated_theory_1} Given a one-dimensional input sequence $\mathbf{x}: \mathbf{N} \rightarrow \mathbf{R}^n$ and a kernel $w: \left\{0, \cdots, k-1\right\} \rightarrow \mathbf{R}$, the dilated convolution function $\left(\mathbf{x}\star_d w\right):
\mathbf{N} \rightarrow \mathbf{R}^n$ is defined as, $
\left(\mathbf{x}\star_d w\right)(s) = \sum_{i=0}^{k-1} w(i)
\mathbf{x}(s-d\times i), $ where $\mathbf{N}$ is the set of natural numbers and $k$ and $d$ are the kernel size and the dilation factor respectively.  Note that with $d=1$, we get the normal convolution operator. In a dilated CNN, the receptive field size will depend on the depth of the network as well as on the choice of $k$ and $d$. 

\subsection{Dilated VolterraNet} 
\label{theory:dvnet}
Now we present a dilated VolterraNet  model by combining the VolterraNet with the dilated CNN model. Given a one-dimensional input sequence $\left\{f_i: \mathcal{M}
\rightarrow \mathbf{R}\right\}$, we will first apply
$\textsf{Corr}_2^{\mathcal{M}}$ and cascaded $\textsf{Corr}_2^{G}$ layers to each point in the sequence independently. The output of a $\textsf{Corr}_2^{G}$ layer is a function $G \rightarrow \mathbf{R}$. Let the output of the last $\textsf{Corr}_2^{G}$ layer be $\left\{g_i: G \rightarrow
\mathbf{R}\right\}$. Then, we discretize the group $G$, to represent each $g_i$ by a vector $\mathbf{x}_i$ (as shown in Fig. \ref{fig2}). {\color{blue} The steps of discretization, i.e., length of $\mathbf{x}_i$, are chosen via grid search in the experimental section. This is analogous to the standard practice in literature \cite{Cohen-ICLR18,esteves2018learning}. Polar coordinates on $G$ are used  to discretize $G$ and then we use the dilated CNN by treating each sample as a vector. This essentially amounts to choosing a uniform grid in the parameter space using Rodrigues vectors \cite{hamilton1866elements}, although more sophisticated techniques can be employed in this context \cite{kurz2017discretization}.} Now, we input
$\left\{\mathbf{g}_i\right\}$ to the Euclidean dilated CNN (since the components of $\mathbf{g}_i$ are real) to construct a dilated VolterraNet framework. { In Fig. \ref{fig2}, we present a schematic of dilated VolterraNet with input $\left\{f_i\right\}$ followed by {\bf $\textsf{Corr}_2^{\mathcal{M}}$ $\rightarrow$\textsf{ReLU}
$\rightarrow$ $\textsf{Corr}_2^{G}$ $\rightarrow$
\textsf{ReLU} $\rightarrow$ $\textsf{Corr}_2^{G}$}. }

A self explanatory schematic diagram of the
dilated VolterraNet architecture is shown in Fig. \ref{fig2}.

\begin{figure}[!ht]
        \centering
                \includegraphics[scale=0.2]{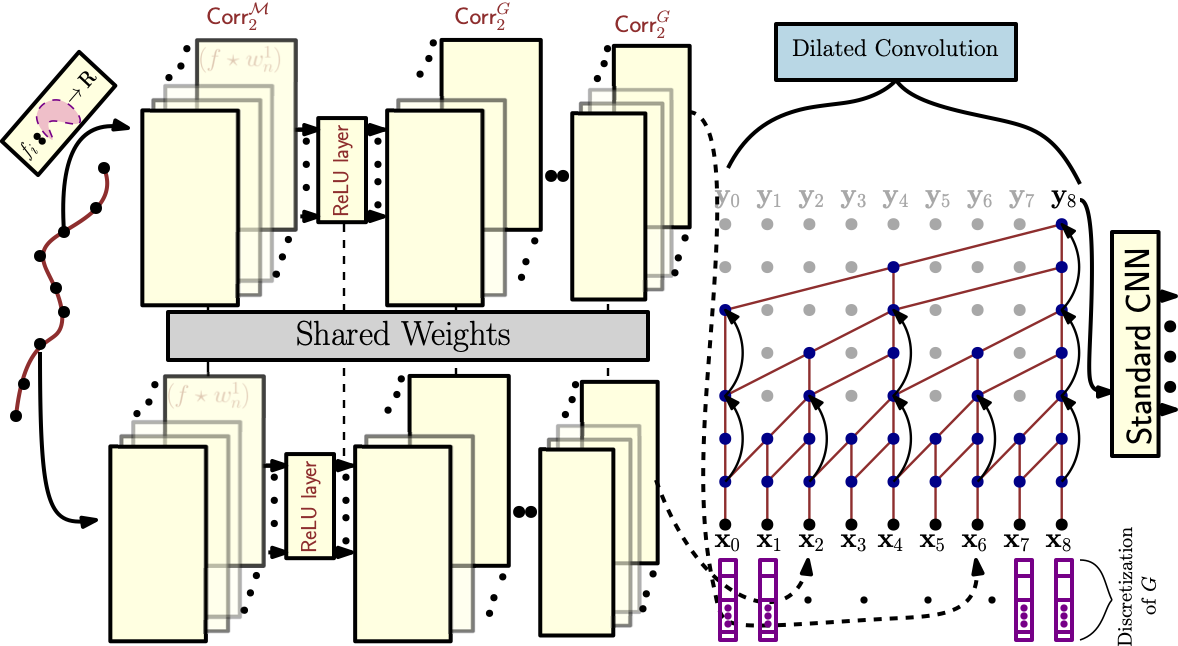}
               \caption{\footnotesize Schematic diagram of dilated VolterraNet}\label{fig2}
\end{figure}

\section{Experiments}\label{results}

In this section, we present experiments on spherical MNIST, atomic
energy and Shrec17 data sets respectively.  We present comparisons of
performance of our VolterraNet to Spherical CNN by Cohen et
al. \cite{Cohen-ICLR18} and Clebsch-Gordan net by Kondor et
al. \cite{kondor2018clebsch}. Further, we also present a separate
comparison with the spherical CNN presented most recently in
\cite{esteves2018learning} on the shrec17 data set. The separate
comparison was necessary due to the fact that the loss function used
in \cite{esteves2018learning} was distinct from the one used in
\cite{Cohen-ICLR18,kondor2018clebsch}.  Finally, we extend the
VolterraNet to a dilated version, a higher order analogue of the
dilated CNN and use it to demonstrate its efficacy in group testing on
diffusion MRI (dMRI) data acquired from movement disorder
patients. The data in this example reside in a product space, $\mathbf{S}^2
\times \mathbf{R}^{+}$, which is a Riemannian homogeneous space distinct from
$\mathbf{S}^2$. This experiment serves as an example demonstrating the ability
of VolterraNet to cope with manifolds other than the sphere.

{\bf Choice of basis:} In our experiments, we have three examples of manifolds, $\mathbf{S}^2$, $\mathbf{S}^2\times \mathbf{R}^+$ and $P_3$. For $\textsf{SO}(3)$, we use the Weigner basis and for $\mathbf{S}^2 \simeq  \textsf{SO}(3)/\textsf{SO}(2)$, we use the induced basis, i.e., the Spherical Harmonics basis. For $\mathbf{S}^2\times \mathbf{R}^+$, we use the product basis of each of the spaces, i.e., Spherical Harmonics for $\mathbf{S}^2$ and the canonical basis for $\mathbf{R}^+$. Since $P_3$ can be written as $\textsf{GL}(3)/\textsf{O}(3)$
, we use the induced basis on $P_3$ which are induced from the canonical basis on $\textsf{GL}(3)$.

In all the experiments we compare the VolterraNet architecture to the state of the
art models, and additionally compare it to an architecture which we call 
Homogeneous CNN (HCNN) which replaces the Volterra non-linear convolutions/correlations with
the correlation operations from definition \ref{theory:def5}. 
We have released an implementation of the VolterraNet architecture with the Spherical MNIST experiment which can be found at, {\color{blue}{\url{https://github.com/cvgmi/volterra-net}}}. 

\subsection{\bf Synthetic Data Experiment: Classification of data on $P_3$}
In this section, we first describe the process of synthesizing
functions $f: P_3 \rightarrow \left[0,1\right]$. In this experiment,
we generated data samples drawn from distinct Gaussian distributions
defined on $P_3$ \cite{cheng2013novel}. Let $\mathcal{X}$ be a $P_3$
valued random variable that follows $\mathcal{N}\left(M,
\sigma\right)$, then, the p.d.f. of $X$ is given by \cite{cheng2013novel}: 
\begin{align}
f_{\mathcal{X}}\left(X;M,\sigma\right) = \frac{1}{C(\sigma)}
\exp\left(-\frac{d^2(M, X)}{2\sigma^2}\right),
\end{align}
where, $d(.,.)$ is the affine invariant geodesic distance on $P_3$ as given by
$d(M, X) = \sqrt{\text{trace}\left(\left(\text{logm}\left(M^{-1}X\right)\right)^2\right)}$.

We first chose two sufficiently spaced apart location parameters $M_1$
and $M_2$ and then for the $i^{th}$ class we generate Gaussian
distributions with location parameters that are perturbations of $M_i$
and with variance $1$. This gives us two clusters in the space of
Gaussian densities on $P_3$, which we will classify using 
HCNN and VolterraNet. In this case, the HCNN network architecture is given by:
{\bf $\textsf{Corr}^{P_3}$ $\rightarrow$ \textsf{ReLU}
$\rightarrow$ $\textsf{Corr}^{\textsf{GL}(3)}$ $\rightarrow$
\textsf{ReLU} $\rightarrow$ $\textsf{Corr}^{\textsf{GL}(3)}$
$\rightarrow$ \textsf{ReLU} $\rightarrow$ \textsf{iL} $\rightarrow$ \textsf{FC}}.
and for VolterraNet the correlation operations are replaced with the corresponding Volterra 
convolutions. 
\begin{table}[!ht]
\begin{center}
%\scalebox{0.85}{
\begin{tabular}{|c|c|c|}
\hline
Model & mean acc. & std. acc. \\
\hline\hline
VolterraNet & $\highest{ 91.50}$ & $\highest{0.08}$ \\
HCNN & $86.50$ & ${ 0.02}$\\
\hline
\end{tabular}
%}
\end{center}
%\vspace{-2em}
\caption{Comparative mean and stdev. on the synthetic data}
\label{results:tab000}
\vspace{-2em}
\end{table}

The data consists of $500$ samples from each class, where each sample is drawn from a Gaussian distribution on $P_3$. The classification accuracies in a ten-fold partition of the data are shown in Table \ref{results:tab000}. In most deep
learning applications, one is used to seeing a high classification
accuracy, but we believe that this can be achieved here as well by
increasing the number of layers and possibly overfitting the data. The
purpose of this synthetic experiment was not to seek an ``optimal''
classification accuracy but to provide a flexible framework which if
``optimally'' tuned can yield a good testing accuracy for data whose
domain is a non-compact Riemannian homogeneous space.

\subsection{Spherical MNIST Data Experiment}

The spherical MNIST data are generated using the scheme described in
\cite{Cohen-ICLR18}. There are two instances of this data, one in which we project MNIST digits on the northern hemisphere (denoted by `NR') and the other where we apply random rotation afterwards (denoted by `R'). The spherical signal is discretized using a bandwidth of 60. %Some sample images for Spherical MNIST are shown in Fig. \ref{fig0}.

We selected the same baseline model as was chosen in
\cite{Cohen-ICLR18}, which is a Euclidean CNN with $5\times 5$ filters and $32, 64, 10$ channels with a stride of $3$ in each layer. This CNN is trained by mapping the digits from the northern hemisphere onto the plane. The Spherical CNN model \cite{Cohen-ICLR18} we used has the
following architecture (as was reported in \cite{Cohen-ICLR18}), {\bf
  $\textsf{Corr}^{\mathbf{S}^2}$ $\,\to\,$ ReLU $\,\to\,$
  $\textsf{Corr}^{\textsf{SO}(3)}$ $\,\to\,$ ReLU $\,\to\,$ FC} with bandwidths $20$, $12$ and the number of channels $20$, $40$ respectively. We used the same architecture for Clebsch-Gordan net as was reported in \cite{kondor2018clebsch}.

For our method, we used a second order Volterra network with the following architecture: {\bf $\textsf{Corr}_2^{\mathbf{S}^2}$
  $\,\to\,$ ReLU $\,\to\,$ iL $\,\to\,$ FC} with bandwidth $30$, $20$
respectively and number of features $25$, $10$ respectively. We chose
a batchsize of $32$ and learning rate of $5\times 10^{-3}$ with ADAM
optimization \cite{kinga2015method}.

We performed three sets of experiments: non rotated training and test sets (denoted by `NR/NR'), non rotated training and randomly rotated test sets (denoted by `NR/R') and randomly rotated both training and test sets (denoted by `R/R'). The comparative results in terms of classification accuracy are shown in Table \ref{tab1}.

\begin{table}[!ht]
\begin{center}
%\scalebox{0.85}{
\begin{tabular}{|l|c|c|c|c|}
\hline
Method & NR/NR & NR/R & R/R & \# params. \\
\hline\hline
Baseline CNN &  97.67 & {22.18} & {12.00} & {68000}\\
Spherical CNN \cite{Cohen-ICLR18} & {95.59} & {94.62} & {93.40} & {58550}\\
Clebsch-Gordan net\cite{kondor2018clebsch}  & {96.00} & {95.86} & {95.80} & {342086}\\
\hline
VolterraNet & $\highest{ 96.72}$ & $\highest{ 96.10}$ & $\highest{96.71}$ &   $\highest{46010}$\\
\hline
\end{tabular}
%}
\end{center}
\caption{Comparison of classification accuracy on Spherical MNIST data}
\label{tab1}
\vspace{-2em}
\end{table}

We can see that the VolterraNet performed better than all the three competing networks for both the `R/R' and 'NR/R' cases. Note that in terms of number of parameters, VolterraNet used {\bf 46010}, while Spherical CNN used
{\bf 58550} and Clebsch-Gordan net used {\bf 342086}. The baseline CNN
used {\bf 68000} parameters. Thus in comparison, we have approximately
an $86\%$ reduction in parameters over the Clebsch-Gordan net with
almost equal or better classification accuracy. In comparison to the
Spherical CNN, we have approximately a $21\%$ reduction in the
parameters over the Spherical CNN while achieving significantly better
performance. This clearly depicts the usefulness of our proposed
VolterraNet in comparison to existing networks used in processing this
type of data in a non-Euclidean domain.

\subsection{3D Shape Recognition Experiment}

We now report results for shape classification using the Shrec17
dataset \cite{yi2017large} which consists of $51300$ 3D models spread
over $55$ classes. This dataset is divided into a $70/10/20$ split for
train/validation/test. Following the method in \cite{Cohen-ICLR18}, we
perturbed the dataset using random rotations. We processed
the dataset as in \cite{Cohen-ICLR18}. Basically, we 
represented each 3D model by a spherical signal using a ray casting
scheme. For each point on the sphere, a ray towards the origin is sent
which collects the ray length, cosine and sine of the surface
angle. Additionally, the convex hull of the 3D shape gives 3 more
channels, which results in 6 input channels. The spherical signal is
discretized using Discoll-Healy grid \cite{driscoll1994computing} grid with
a bandwidth of 128. %Some sample models for Shrec17 are shown in Fig. \ref{results:fig_shrec}.

%\begin{figure}[!ht]
%   \centering
%      \includegraphics[scale=0.45]{../Figures/shrec17_montage.png}
%      \caption{Sample models for Shrec17}
%      \label{results:fig_shrec}
%\end{figure}

The Spherical CNN model \cite{Cohen-ICLR18} we used has the following
architecture (as was reported in \cite{Cohen-ICLR18}): {\bf
  $\textsf{Corr}^{\mathbf{S}^2}$ $\,\to\,$ BN $\,\to\,$ ReLU $\,\to\,$
  $\textsf{Corr}^{\textsf{SO}(3)}$ $\,\to\,$ BN $\,\to\,$ ReLU
  $\,\to\,$ $\textsf{Corr}^{\textsf{SO}(3)}$ $\,\to\,$ BN $\,\to\,$
  ReLU $\,\to\,$ FC} with bandwidths $32$, $22$ and $7$ and the number
of channels $50$, $70$ and $350$ respectively. We used the same
architecture for Clebsch-Gordan net as was reported in
\cite{kondor2018clebsch}.

In our method, we used a second order Volterra network with the
following architecture: {\bf $\textsf{Corr}^{\mathbf{S}^2}$ $\,\to\,$
  BN $\,\to\,$ ReLU $\,\to\,$ $\textsf{Corr}_2^{\textsf{SO}(3)}$
  $\,\to\,$ BN $\,\to\,$ ReLU $\,\to\,$ iL $\,\to\,$ FC} with
bandwidths $10$, $8$, $8$ respectively and number of features $60$,
$80$, $100$ respectively. We chose a batch size of $100$ and a learning rate of
$5\times 10^{-3}$ with ADAM optimization \cite{kinga2015method}. Table \ref{tab2} summarizes comparison of VolterraNet with other existing deep network architectures that reported results on this data in literature. From this table, it is evident that VolterraNet almost always yields classification accuracy results within the top three methods, while having the best parameter efficiency. 

\begin{table*}
\begin{center}
%\scalebox{0.85}{
\begin{tabular}{|l|c|c|c|c|c|c|}
\hline
Method & P@N & R@N & F1@N & mAP & NDCG & \# params. \\
\hline\hline
Tasuma\_ReVGG \cite{tatsuma2009multi} & 0.70 & 0.76 & 0.72 & 0.69 & 0.78 & 3M\\ 
Furuya\_DLAN \cite{furuya2016deep} & 0.81 & 0.68 & 0.71 & 0.65 & 0.75 & 8.4M  \\ 
SHREC16-bai\_GIFT \cite{qi2016volumetric} & 0.68 & 0.66 & 0.66 & 0.60 & 0.73 & 36M \\ 
Deng\_CM-VGG5-6DB  & 0.41 & 0.70 & 0.47 & 0.52 & 0.62 & - \\ 
Spherical CNN \cite{Cohen-ICLR18} & 0.70 & 0.71 & 0.69 & 0.67 & 0.76 & 1.4Mil \\ 
Clebsch-Gordan net\cite{kondor2018clebsch}  & 0.70 & 0.72 & 0.70 & 0.68 & 0.76 & -\\ \hline
Ours (VolterraNet) & $\highest{ 0.71}$ (2nd) & $\highest{0.70}$ (3rd) & $\highest{0.70}$ (3rd) & $\highest{0.67}$ (3rd) & $\highest{0.75}$ (4th) & $\highest{396297}$\\
\hline \hline
\cite{esteves2018learning} (w triplet loss) & 0.72 & 0.74 & 0.69 & - & - & 0.5M\\
Ours (w triplet loss) & $\highest{0.73}$ & $\highest{0.74}$ & $\highest{0.70}$ & $\highest{0.68}$ & $\highest{0.76}$ & $\highest{396297}$\\ \hline
\end{tabular}
%}
\end{center}
\caption{Comparison results in terms of classification accuracy on
  the shrec17 data}
\label{tab2}
\vspace{-1.8em}
\end{table*}

{\bf Comparison with Esteves et al. \cite{esteves2018learning}}: We
also compared our VolterraNet with recent work of Esteves et
al. \cite{esteves2018learning} using an extra in-batch triplet loss
\cite{schroff2015facenet} (as used in Esteves et
al. \cite{esteves2018learning}). We show the comparison results in
Table \ref{tab2} (last two rows), which clearly shows
that, \begin{inparaenum}[\bfseries (a)] \item The VolterraNet outperforms the network
  in \cite{esteves2018learning} (which is the state-of-the-art
  algorithm in terms of parameter efficiency). \item The triplet loss boosts the performance of
  VolterraNet relative to the baseline loss of cross entropy. \end{inparaenum}
%\vspace{-0.5em}

\subsection{Regression Experiment: Prediction of Atomic Energy}

Here, we report the application of our VolterraNet to the QM7 dataset
\cite{blum2009970,rupp2012fast}, where the goal is to regress over
atomization energies of molecules given atomic positions $(\mathbf{p}_i)$ and
charges $(\mathbf{z}_i)$. Each molecule consists of at most $23$ atoms and the
molecules are of $5$ types (C, N, O, S, H). We use the Coulomb Matrix
(CM) representation proposed by\cite{rupp2012fast}, which is rotation
and translation invariant but not permutation invariant.  We used a
similar experimental setup to that described in \cite{Cohen-ICLR18}
for this regression problem. We define a sphere $\mathbf{S}_i$ around $\mathbf{p}_i$ for
each $i^{th}$ atom. We define the potential functions 
\begin{align}
U_{\mathbf{z}}(\mathbf{x}) =
\sum_{i \neq j, \mathbf{z}_j = \mathbf{z}} \frac{\mathbf{z}_i^t\mathbf{z}}{\|\mathbf{x}-\mathbf{p}_i\|},  
\end{align}
for every $\mathbf{z}$ and
for every $\mathbf{x}$ on the sphere $\mathbf{S}_i$. This yields a spherical signal
consisting of $5$ features which were discretized using the
Discoll-Healy grid \cite{driscoll1994computing}  with a bandwidth
of 20. For the VolterraNet, we used one $\mathbf{S}^2$ and
$\textsf{SO}(3)$ second order Volterra block with bandwidths $12, 8, 8,
4$ and number of features $8,10,20,50$ respectively.

We compute the loss and report it in Table \ref{tab4}. We can see that
VolterraNet performs better than the competing methods. For Spherical
CNN \cite{Cohen-ICLR18} and Clebsch-Gordan
net\cite{kondor2018clebsch}, we used similar architectures as
described in the respective papers. Spherical CNN \cite{Cohen-ICLR18}
and Clebsch-Gordan net\cite{kondor2018clebsch} use $1.4M$ and $1.1M$
parameters respectively, while the VolterraNet used $128460$, nearly an order
of magnitude reduction of parameters, while achieving the best classification accuracy. 
This illustrates the parameter efficiency gains that we get from using a higher order correlation, 
a richer feature, in the VolterraNet. 

\begin{table}
\begin{center}
%\scalebox{0.85}{
\begin{tabular}{|l|c|}
\hline
Method & MSE  \\
\hline\hline
MLP/ Random CM \cite{montavon2012learning} & {5.96} \\ 
LGIKA (RF) \cite{raj2016local} & 10.82 \\ 
RBF Kernels/ Random CM \cite{montavon2012learning} & 11.42 \\ 
RBF Kernels/ Sorted CM \cite{montavon2012learning}& 12. 59 \\ 
MLP/ Sorted CM \cite{montavon2012learning} & 16.06 \\ 
Spherical CNN \cite{Cohen-ICLR18} & 8.47 \\ 
Clebsch-Gordan net\cite{kondor2018clebsch} & 7.97 \\ \hline
Ours (VolterraNet) & $\highest{5.92}$ (1st) \\
\hline
\end{tabular}
%}
\end{center}
%\vspace{-1em}
\caption{Comparison results on atomic energy prediction}
\label{tab4}
\vspace{-2.5em}
\end{table}
\vspace{-0.5em}

\subsection{Network architecture for dMRI data Using Dilated VolterraNet}
\label{architecture}
Diffusion MRI (dMRI) is an imaging modality that non-invasively measures the diffusion of water molecules in tissue samples being imaged. It serves as an interesting example of our framework since dMRI 
data can naturally be described by functions on a Riemannian homogeneous space. In this
section we describe the dMRI data and its processing using the framework presented in this
paper, which will help the reader understand the results of the following subsections.

In each voxel of a dMRI data set, the signal magnitude is represented by a real
number along each gradient magnetic field over a hemi-sphere of
directions in 3D. Hence, in each voxel, we have a function $f:
\mathbf{S}^2\times \mathbf{R}^{+} \rightarrow \mathbf{R}$. The
proposed network architecture has two components:  {\it
  intra-voxel layers} and {\it inter-voxel layers}. The {\it intra-voxel layers} extract features from each voxels, while the {\it inter-voxel layers} use dilated convolution to capture the interaction between extracted features. In our application in the next section we extract a sequence of voxels lying along a nerve fiber bundle in the brain known to be affected in Parkinson disease. Hence we have a sequence of functions along the  fiber bundle $\{f_i : \mathbf{S}^2 \times \mathbf{R}^+ \rightarrow \mathbf{R}\}$, making the application of the dilated VolterraNet in section \ref{theory:dvnet} appropriate. 
  
\subsubsection{Extracting intra-voxel features}

We extract intra-voxel features (independently) from each voxel. 
As mentioned before, in each voxel we have a function
$f: \mathbf{S}^2\times \mathbf{R}^{+} \rightarrow \mathbf{R}$. Since
$\mathbf{S}^2\times \mathbf{R}^{+}$ is a Riemannian homogeneous space
(endowed with the product metric), we will use a cascade of the Volterra
correlation layers defined earlier (with standard non-linearity between layers) to extract features
which are {\it equivariant} to the action of $\mathsf{SO}(3) \times
\left(\mathbf{R} \setminus \left\{0\right\}\right)$. These features are extracted independently within each voxel. Observe that this equivariance property is natural in the context of
dMRI data. Since in each voxel of the dMRI data, the signal is
acquired in different directions (in 3D), we want the features to be
equivariant to the 3D rotations and scaling. 

\subsubsection{Extracting inter-voxel features}

After the extraction of the intra-voxel features (which are
equivariant to the action of $G$), we seek to derive features based on
the interactions \textit{between} the voxels. Here we use the standard dilated convolution (as described in \ref{dilated_theory_1}) layers to capture the interaction between features extracted from voxels. 

Now, we are ready to give the details of the data used for the experiment of our proposed Dilated-VolterraNet.  For this experiment, we used a second order Dilated-VolterraNet with
$3$ dilated layers of kernel size $(5\times 5)$ and dilation factors of
$1$, $2$ and $4$ respectively.  

\subsection{Dilated VolterraNet Experiment: Group testing on movement disorder patients}

This dMRI data was collected  from  $50$ PD patients and $44$ controls at the University of Florida and are accessible via request from the NIH-NINDS Parkinson's Disease Biomarker Program portal \url{https://pdbp.ninds.nih.gov/}. All images were collected using a 3.0 T MR scanner (Philips Achieva) and 32-channel quadrature volume head coil. The parameters of the diffusion imaging acquisition sequence were as follows: gradient directions = 64, b-values = 0/1000 s/mm2, repetition
time =7748 ms, echo time = 86 ms, flip angle = $90^{\circ}$, field of view = $224 \times 224$ mm, matrix size = $112 \times 112$, number of contiguous axial slices = 60, slice thickness = 2 mm. Eddy current correction was applied to each data set by using standard motion correction techniques.
\begin{wrapfigure}{r}{3.8cm}
    \vspace{-1em}
    \begin{center}
      \includegraphics[scale=0.18]{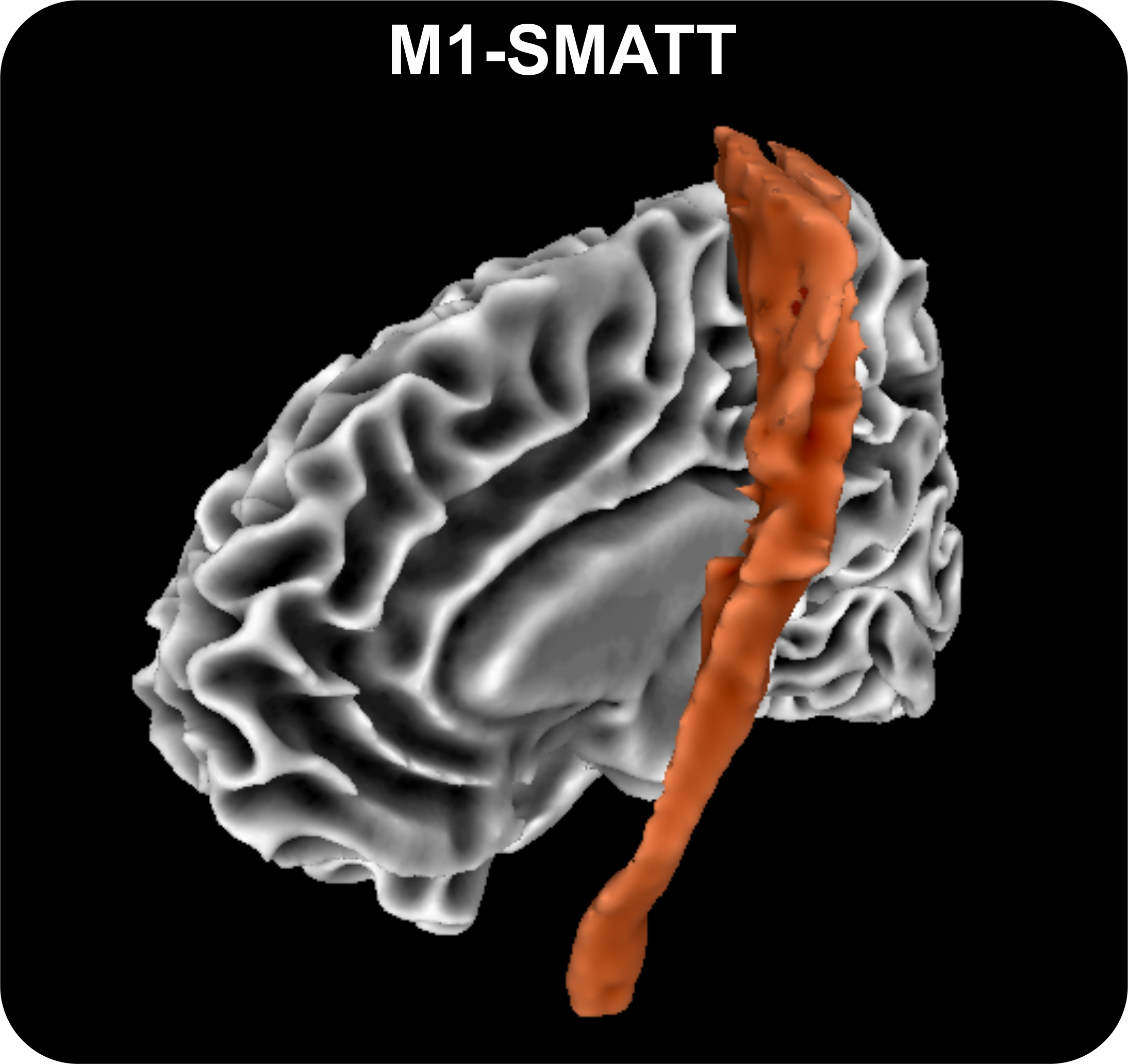}
      \vspace{-1em}
      \caption{M1 Template \cite{chakraborty2018statistical}}
      \label{results:fig1}
      \end{center}
      \vspace{-1em}
\end{wrapfigure}

We first extracted the sensory motor area tracts called M1 fiber tracts (as shown in Fig.\ref{results:fig1}) using the FSL software \cite{archer2017template} from both the left (`LM1') and right hemispheres (`RM1'). We applied the Dilated-Volterra to the raw signal measurements along the fiber tracts. Our model was trained on both the control (normal subjects) group and the PD group data sets, i.e., we learned two Dilated VolterraNet models, one each for the control and the PD groups
respectively. Using the method from \cite{triacca2016measuring} we
compute the distance between these two models, denoted by
$d$. Now we permute the class labels between the classes,
retrain two models and compute the network distance $d_j$. If there
are significant differences between the classes we should expect that
$d > d_j$. We repeat this experiment for $j=1,\ldots, 1000$
and let $p$ be the proportion of experiments for which $d
\leq d_j$. This is a permutation test of the null hypothesis: there is
no significant difference between the tract models learned from the
two different classes. We also performed ablation studies with regards to
the order of the model in the Dilated-VolterraNet to study the effect of higher order convolutions. We used the
following architecture {\bf $\textsf{Corr}_2^{\mathbf{S}^2}$ $\,\to\,$
  BN $\,\to\,$ ReLU $\,\to\,$ $\textsf{Corr}_2^{\textsf{SO}(3)}$
  $\,\to\,$ BN $\,\to\,$ ReLU} as our baseline model and then replaced
{\bf $\textsf{Corr}_2^{\textsf{SO}(3)}$} and {\bf
  $\textsf{Corr}_2^{\mathbf{S}^2}$} to {\bf
  $\textsf{Corr}^{\textsf{SO}(3)}$} and {\bf
  $\textsf{Corr}^{\mathbf{S}^2}$} respectively in an alternating
fashion. The ablation study result is presented in Table
\ref{tab6}. The `N' in {\bf $\textsf{Corr}^{\textsf{SO}(3)}$} ({\bf
  $\textsf{Corr}^{\mathbf{S}^2}$}) indicates that we used second order
for the respective convolution operator. The table shows that a
second order representation in later layers is very useful and hence a
model with {\bf $\textsf{Corr}^{\textsf{SO}(3)}$} performs poorly but
a model with {\bf $\textsf{Corr}^{\mathbf{S}^2}$} and {\bf
  $\textsf{Corr}_2^{\textsf{SO}(3)}$} performs as good as the model with
both second order kernels. Both models reject the null hypothesis with
$95\%$ confidence. 

We compared our dilated VolterraNet with the standard (no dilation) VoletrraNet and as expected we needed $\approx 1.5\times $ parameters in case of standard VolterraNet to achieve p-values of $0.03$ and $0.04$ for LM1 and RM1 respectively, which is similar in performance to its dilated counterpart. Additionally, we compared our network's performance to the performance of a similar dMRI architecture (recurrent model) namely, the SPD-SRU \cite{chakraborty2018statistical} and the baseline model used for comparison in \cite{chakraborty2018statistical} (see section 5.2 of \cite{chakraborty2018statistical} for details on the baseline model). 
We found that the baseline method yielded a p-value of $0.17$ and $0.34$ respectively for 'LM1' and 'RM1'. Whereas, the SPD-SRU architecture yielded a p-values of $0.01$ and $0.032$ respectively. We can conclude that both using standard and Dialted VolterraNet we can reject the null hypothesis with $95 \%$ confidence whereas Dilated VoletrraNet can achieve the statistically significant result with $\approx 33\%$ reduction in number of parameters compared to its standard  counterpart.

\begin{table}
\begin{center}
\scalebox{0.85}{
\begin{tabular}{|c|c|c|c|}
\hline
{\bf $\textsf{Corr}^{\mathbf{S}^2}$}  & {\bf $\textsf{Corr}^{\textsf{SO}(3)}$} & \multicolumn{2}{c|}{p-val.} \\
\hline
\cline{3-4} 
& & `LM1' & `RM1' \\\hline
N & N & $\highest{0.01}$ & $\highest{0.02}$ \\
Y & N & $\highest{0.04}$ & $\highest{0.03}$\\
N & Y & 0.13 & 0.24\\
Y & Y & 0.15 &  0.26\\
\hline
\end{tabular}
}
\end{center}
\caption{Ablation studies for Dilated-VolterraNet}
\label{tab6}
\vspace{-2em}
\end{table}

\section{Conclusions}\label{conc}
In this paper, we presented a novel generalization of CNNs to 
non-Euclidean domains specifically, Riemannian homogeneous 
spaces. More precisely, we introduced higher
order convolutions -- represented using a Volterra series -- on Riemannian 
homogeneous spaces. 
We call our network a Volterra homogeneous CNN abbreviated as VolterraNet. 
The
salient contributions of our work are: (i) A proof of equivariance of
higher order convolutions to group actions on homogeneous Riemannian
manifolds. Proofs of generalized Linear Shift Invariant (equivariant) and Nonlinear
Shift Invariant (eqivariant) theorems for correlations and Volterra series defined on
Riemannian homogeneous spaces. 
(ii) We prove that second order Volterra convolutions can be
expressed as a cascade of convolutions. This allows for efficient
implementation of second-order Volterra representation used in the VolterraNet.
(iii) In support of our conjecture on the reduced number of
parameters, real data experiments empirically
demonstrate that VolterraNet requires less number of parameters to achieve the baseline accuracy of classification in comparison to both 
Spherical-CNN and Clebsch-Gordan net. (iv) We also presented a dilated VolterraNet that was shown to be effective on a group testing experiment on movement disorder patients.
Our future work will be focused on performing
more real data experiments to demonstrate the power of VolterraNet for a
variety of data domains that are Riemannian homogeneous spaces.
\section*{Acknowledgements}
This research was in part funded by the NSF grant IIS-1724174 to BCV. We thank Professor David Vaillancourt of the University of Florida, Department of Applied Physiology and Kinesiology for providing us with the diffusion MRI scans used in this work.

\bibliographystyle{ieeetran}
\bibliography{references,references_v}

\begin{thebibliography}{10}
\providecommand{\url}[1]{#1}
\csname url@samestyle\endcsname
\providecommand{\newblock}{\relax}
\providecommand{\bibinfo}[2]{#2}
\providecommand{\BIBentrySTDinterwordspacing}{\spaceskip=0pt\relax}
\providecommand{\BIBentryALTinterwordstretchfactor}{4}
\providecommand{\BIBentryALTinterwordspacing}{\spaceskip=\fontdimen2\font plus
\BIBentryALTinterwordstretchfactor\fontdimen3\font minus
  \fontdimen4\font\relax}
\providecommand{\BIBforeignlanguage}[2]{{%
\expandafter\ifx\csname l@#1\endcsname\relax
\typeout{** WARNING: IEEEtran.bst: No hyphenation pattern has been}%
\typeout{** loaded for the language `#1'. Using the pattern for}%
\typeout{** the default language instead.}%
\else
\language=\csname l@#1\endcsname
\fi
#2}}
\providecommand{\BIBdecl}{\relax}
\BIBdecl

\bibitem{lecun1998gradient}
Y.~LeCun, L.~Bottou, Y.~Bengio, and P.~Haffner, ``Gradient-based learning
  applied to document recognition,'' \emph{Proceedings of the {IEEE}}, vol.~86,
  no.~11, pp. 2278--2324, 1998.

\bibitem{krizhevsky2009learning}
A.~Krizhevsky and G.~Hinton, ``Learning multiple layers of features from tiny
  images,'' Tech. Report, University of Toronto, Tech. Rep., 2009.

\bibitem{volterra2005theory}
V.~Volterra, \emph{Theory of functionals and of integral and
  integro-differential equations}.\hskip 1em plus 0.5em minus 0.4em\relax
  Courier Corporation, 2005.

\bibitem{Worrall16}
E.~Worrall, J.~Garbin, D.~Turmukhambetov, and J.~Brostow, ``Harmonic networks:
  Deep translation and rotation equivariance,'' in \emph{Proceedings of the
  {IEEE CVPR}}.\hskip 1em plus 0.5em minus 0.4em\relax {IEEE}, 2017, pp.
  5026--5037.

\bibitem{Cohen-ICLR18}
T.~Cohen, M.~Geiger, J.~Koehler, and M.~Welling, ``Spherical {CNN}s,'' in
  \emph{Proceedings of {ICLR}}.\hskip 1em plus 0.5em minus 0.4em\relax {JMLR},
  2018.

\bibitem{Cohen-ICML17}
------, ``Convolutional networks for spherical signals,'' in \emph{Proceedings
  of {ICML}}.\hskip 1em plus 0.5em minus 0.4em\relax {JMLR}, 2017.

\bibitem{esteves2018learning}
C.~Esteves, C.~Allen-Blanchette, A.~Makadia, and K.~Daniilidis, ``Learning so
  (3) equivariant representations with spherical cnns,'' in \emph{Proceedings
  of the European Conference on Computer Vision (ECCV)}, 2018, pp. 52--68.

\bibitem{kondor2018clebsch}
R.~Kondor, Z.~Lin, and S.~Trivedi, ``Clebsch-gordan nets: a fully fourier space
  spherical convolutional neural network,'' \emph{arXiv preprint
  arXiv:1806.09231}, 2018.

\bibitem{esteves2017polar}
C.~Esteves, C.~Allen-Blanchette, X.~Zhou, and K.~Daniilidis, ``Polar
  transformer networks,'' \emph{arXiv preprint arXiv:1709.01889}, 2017.

\bibitem{jaderberg2015spatial}
M.~Jaderberg, K.~Simonyan, A.~Zisserman \emph{et~al.}, ``Spatial transformer
  networks,'' in \emph{Advances in neural information processing systems
  ({NIPS})}, 2015, pp. 2017--2025.

\bibitem{kondor2018generalization}
R.~Kondor and S.~Trivedi, ``On the generalization of equivariance and
  convolution in neural networks to the action of compact groups,'' \emph{arXiv
  preprint arXiv:1802.03690}, 2018.

\bibitem{cohen2018general}
T.~Cohen, M.~Geiger, and M.~Weiler, ``A general theory of equivariant cnns on
  homogeneous spaces,'' \emph{arXiv preprint arXiv:1811.02017}, 2018.

\bibitem{driscoll1994computing}
J.~R. Driscoll and D.~M. Healy, ``Computing fourier transforms and convolutions
  on the 2-sphere,'' \emph{Advances in applied mathematics}, vol.~15, no.~2,
  pp. 202--250, 1994.

\bibitem{kumar2009volterrafaces}
R.~Kumar, A.~Banerjee, and B.~C. Vemuri, ``Volterrafaces: Discriminant analysis
  using volterra kernels,'' 2009.

\bibitem{kumar2012trainable}
R.~Kumar, A.~Banerjee, B.~C. Vemuri, and H.~Pfister, ``Trainable convolution
  filters and their application to face recognition,'' \emph{IEEE transactions
  on pattern analysis and machine intelligence}, vol.~34, no.~7, pp.
  1423--1436, 2012.

\bibitem{hakim1991volterra}
N.~Hakim, J.~Kaufman, G.~Cerf, and H.~Meadows, ``Volterra characterization of
  neural networks,'' in \emph{Signals, Systems and Computers, 1991. 1991
  Conference Record of the Twenty-Fifth Asilomar Conference on}.\hskip 1em plus
  0.5em minus 0.4em\relax IEEE, 1991, pp. 1128--1132.

\bibitem{zoumpourlis2017non}
G.~Zoumpourlis, A.~Doumanoglou, N.~Vretos, and P.~Daras, ``Non-linear
  convolution filters for cnn-based learning,'' in \emph{Computer Vision
  (ICCV), 2017 IEEE International Conference on}.\hskip 1em plus 0.5em minus
  0.4em\relax IEEE, 2017, pp. 4771--4779.

\bibitem{helgason1962differential}
S.~Helgason, \emph{Differential geometry and symmetric spaces}.\hskip 1em plus
  0.5em minus 0.4em\relax Academic press, 1962, vol.~12.

\bibitem{banerjee2019dmr}
M.~Banerjee, R.~Chakraborty, D.~Archer, D.~Vaillancourt, and B.~C. Vemuri,
  ``Dmr-cnn: A cnn tailored for dmr scans with applications to pd
  classification,'' in \emph{2019 IEEE 16th International Symposium on
  Biomedical Imaging (ISBI 2019)}.\hskip 1em plus 0.5em minus 0.4em\relax IEEE,
  2019, pp. 388--391.

\bibitem{dummit2004abstract}
D.~S. Dummit and R.~M. Foote, \emph{Abstract algebra}.\hskip 1em plus 0.5em
  minus 0.4em\relax Wiley Hoboken, 2004, vol.~3.

\bibitem{hewitt2012abstract}
E.~Hewitt and K.~A. Ross, \emph{Abstract Harmonic Analysis: Volume I Structure
  of Topological Groups Integration Theory Group Representations}.\hskip 1em
  plus 0.5em minus 0.4em\relax Springer Science \& Business Media, 2012, vol.
  115.

\bibitem{michor2008topics}
P.~W. Michor, \emph{Topics in differential geometry}.\hskip 1em plus 0.5em
  minus 0.4em\relax American Mathematical Soc., 2008, vol.~93.

\bibitem{hackbusch2007tensor}
W.~Hackbusch and B.~N. Khoromskij, ``Tensor-product approximation to operators
  and functions in high dimensions,'' \emph{Journal of Complexity}, vol.~23,
  no. 4-6, pp. 697--714, 2007.

\bibitem{elman1990finding}
J.~L. Elman, ``Finding structure in time,'' \emph{Cognitive science}, vol.~14,
  no.~2, pp. 179--211, 1990.

\bibitem{bai2018convolutional}
S.~Bai, J.~Z. Kolter, and V.~Koltun, ``Convolutional sequence modeling
  revisited,'' 2018.

\bibitem{hamilton1866elements}
W.~R. Hamilton, \emph{Elements of quaternions}.\hskip 1em plus 0.5em minus
  0.4em\relax Longmans, Green, \& Company, 1866.

\bibitem{kurz2017discretization}
G.~Kurz, F.~Pfaff, and U.~D. Hanebeck, ``Discretization of so (3) using
  recursive tesseract subdivision,'' in \emph{2017 IEEE International
  Conference on Multisensor Fusion and Integration for Intelligent Systems
  (MFI)}.\hskip 1em plus 0.5em minus 0.4em\relax IEEE, 2017, pp. 49--55.

\bibitem{cheng2013novel}
G.~Cheng and B.~C. Vemuri, ``A novel dynamic system in the space of spd
  matrices with applications to appearance tracking,'' \emph{{SIAM} journal on
  imaging sciences}, vol.~6, no.~1, pp. 592--615, 2013.

\bibitem{kinga2015method}
D.~Kinga and J.~B. Adam, ``A method for stochastic optimization,'' in
  \emph{International Conference on Learning Representations (ICLR)}, vol.~5,
  2015.

\bibitem{yi2017large}
L.~Yi, L.~Shao, M.~Savva, H.~Huang, Y.~Zhou, Q.~Wang, B.~Graham, M.~Engelcke,
  R.~Klokov, V.~Lempitsky \emph{et~al.}, ``Large-scale 3d shape reconstruction
  and segmentation from shapenet core55,'' \emph{arXiv preprint
  arXiv:1710.06104}, 2017.

\bibitem{tatsuma2009multi}
A.~Tatsuma and M.~Aono, ``Multi-fourier spectra descriptor and augmentation
  with spectral clustering for 3d shape retrieval,'' \emph{The Visual
  Computer}, vol.~25, no.~8, pp. 785--804, 2009.

\bibitem{furuya2016deep}
T.~Furuya and R.~Ohbuchi, ``Deep aggregation of local 3d geometric features for
  3d model retrieval.'' in \emph{BMVC}, 2016, pp. 121--1.

\bibitem{qi2016volumetric}
C.~R. Qi, H.~Su, M.~Nie{\ss}ner, A.~Dai, M.~Yan, and L.~J. Guibas, ``Volumetric
  and multi-view cnns for object classification on 3d data,'' in
  \emph{Proceedings of the IEEE conference on computer vision and pattern
  recognition}, 2016, pp. 5648--5656.

\bibitem{schroff2015facenet}
F.~Schroff, D.~Kalenichenko, and J.~Philbin, ``Facenet: A unified embedding for
  face recognition and clustering,'' in \emph{Proceedings of the IEEE
  conference on computer vision and pattern recognition}, 2015, pp. 815--823.

\bibitem{blum2009970}
L.~C. Blum and J.-L. Reymond, ``970 million druglike small molecules for
  virtual screening in the chemical universe database gdb-13,'' \emph{Journal
  of the American Chemical Society}, vol. 131, no.~25, pp. 8732--8733, 2009.

\bibitem{rupp2012fast}
M.~Rupp, A.~Tkatchenko, K.-R. M{\"u}ller, and O.~A. Von~Lilienfeld, ``Fast and
  accurate modeling of molecular atomization energies with machine learning,''
  \emph{Physical review letters}, vol. 108, no.~5, p. 058301, 2012.

\bibitem{montavon2012learning}
G.~Montavon, K.~Hansen, S.~Fazli, M.~Rupp, F.~Biegler, A.~Ziehe, A.~Tkatchenko,
  A.~V. Lilienfeld, and K.-R. M{\"u}ller, ``Learning invariant representations
  of molecules for atomization energy prediction,'' in \emph{Advances in Neural
  Information Processing Systems}, 2012, pp. 440--448.

\bibitem{raj2016local}
A.~Raj, A.~Kumar, Y.~Mroueh, P.~T. Fletcher, and B.~Sch{\"o}lkopf, ``Local
  group invariant representations via orbit embeddings,'' \emph{arXiv preprint
  arXiv:1612.01988}, 2016.

\bibitem{chakraborty2018statistical}
R.~Chakraborty, C.-H. Yang, X.~Zhen, M.~Banerjee, D.~Archer, D.~Vaillancourt,
  V.~Singh, and B.~C. Vemuri, ``Statistical recurrent models on manifold valued
  data,'' \emph{ArXiv e-prints}, 2018.

\bibitem{archer2017template}
D.~B. Archer, D.~E. Vaillancourt, and S.~A. Coombes, ``A template and
  probabilistic atlas of the human sensorimotor tracts using diffusion mri,''
  \emph{Cerebral Cortex}, vol.~28, no.~5, pp. 1685--1699, 2017.

\bibitem{triacca2016measuring}
U.~Triacca, ``Measuring the distance between sets of {ARMA} models,''
  \emph{Econometrics}, vol.~4, no.~3, p.~32, 2016.

\end{thebibliography}

%\begin{comment}
\vskip -3\baselineskip plus -1fil
\begin{IEEEbiography}
[{\includegraphics[width=1in,height=1.25in,clip,keepaspectratio]{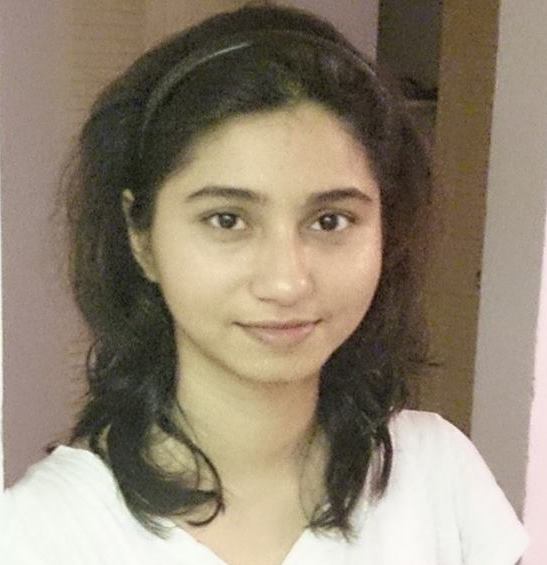}}]{Monami Banerjee} received her Ph.D.~in computer science from the Univ.~of Florida in 2018. She is currently a research staff member at Facebook Oculus, Menlo Park. Her research interests lie at the intersection of Geometry, Computer Vision and Medical Image Anlaysis.
\end{IEEEbiography}
\vskip -3\baselineskip plus -1fil
\begin{IEEEbiography}
[{\includegraphics[width=1in,height=1.25in,clip,keepaspectratio]{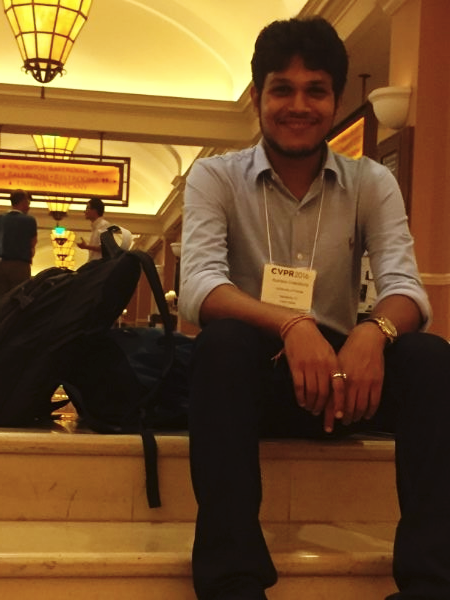}}]{Rudrasis Chakraborty}
received his Ph.D.~in computer science from the Univ.~of Florida in 2018. He is currently a post doctoral researcher at UC Berkeley. His research interests lie in the intersection of Geometry, ML and Computer Vision.
\end{IEEEbiography}
\vskip -2\baselineskip plus -1fil
\begin{IEEEbiography}
[{\includegraphics[width=1in,height=1.25in,clip,keepaspectratio]{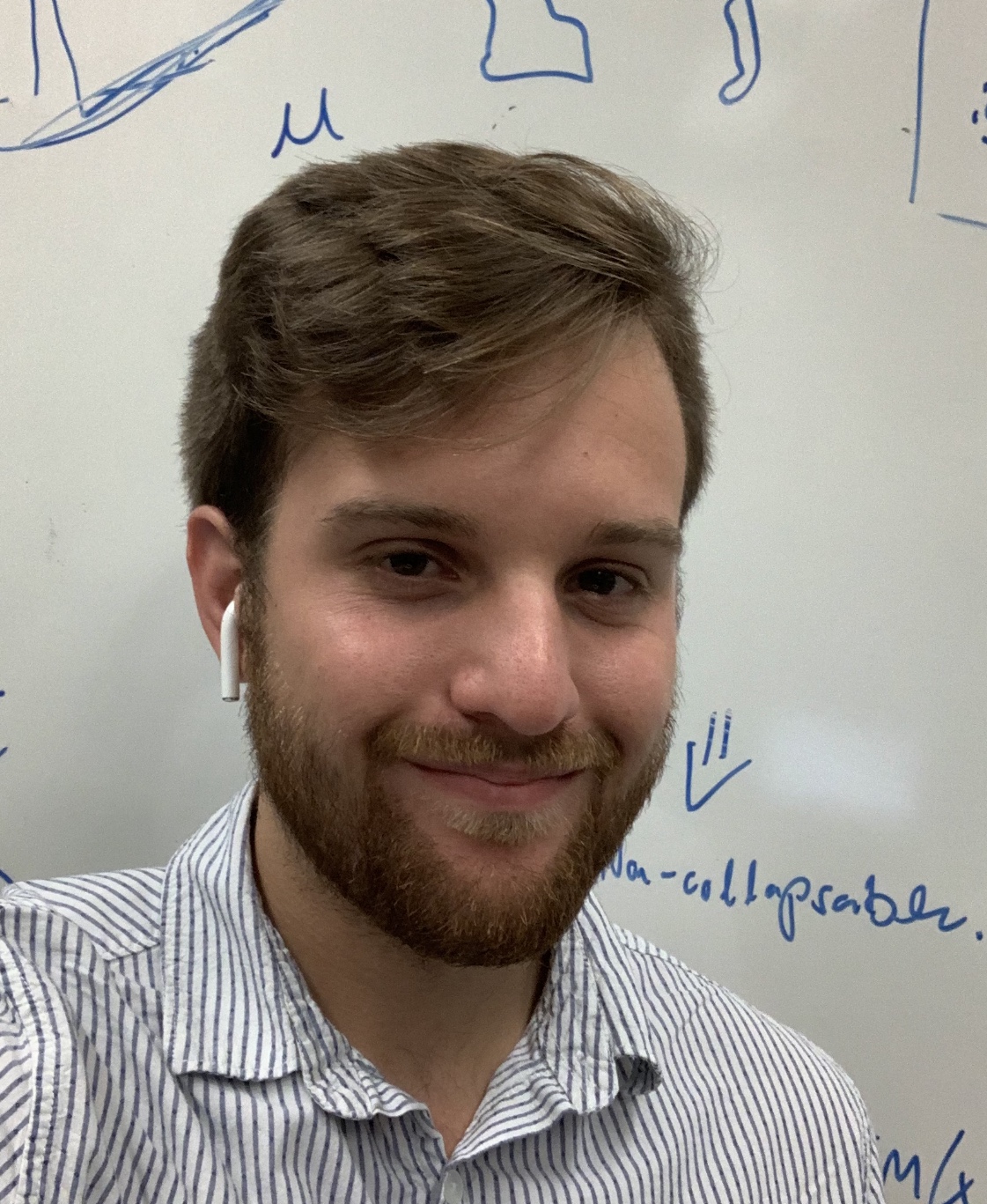}}]{Jose Bouza} is a fourth year Mathematics and Computer Science undergraduate at the University of Florida. His primary interests encompass computer vision and applied topology. 
\end{IEEEbiography}
\vskip -2\baselineskip plus -1fil
\begin{IEEEbiography}
[{\includegraphics[width=1in,height=1.25in,clip,keepaspectratio]{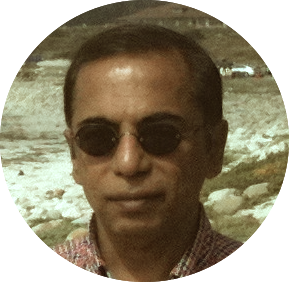}}]{Baba C. Vemuri}
received his PhD in Electrical and Computer Engineering from the University of Texas at Austin. Currently, he holds the Wilson and Marie Collins professorship in Engineering at the University of Florida and is a full professor in the Department of Computer and Information Sciences and Engineering. His research interests include Statistical Analysis of Manifold-valued Data, Medical Image Computing, Computer Vision and Machine Learning.  He is a recipient of the IEEE Technical Achievement Award (2017) and is a Fellow of the IEEE (2001) and the ACM (2009). 
\end{IEEEbiography}
%\vskip -2\baselineskip plus -1fil
%\end{comment}

\end{document}